\documentclass[lettersize,journal]{IEEEtran}
\usepackage{amsmath,amsfonts}
\usepackage{array}
\usepackage[caption=false,subrefformat=parens]{subfig}
\usepackage{textcomp}
\usepackage{stfloats}
\usepackage{url}
\usepackage{verbatim}
\usepackage{graphicx}
\usepackage{cite}

\usepackage{amssymb}
\usepackage{amsthm}
\usepackage{hyperref}
\usepackage[ruled]{algorithm2e}
\usepackage{bm}
\usepackage{multirow}
\usepackage{tabularx}
\usepackage{multicol}
\usepackage{booktabs}
\usepackage{threeparttable}
\usepackage{mathtools}
\usepackage{xcolor}

\DeclarePairedDelimiter\sbrac{(}{)}

\DeclarePairedDelimiter\bbrac{\{}{\}}
\newcommand{\sbr}[1]{\sbrac*{#1}}

\newcommand{\bbr}[1]{\bbrac*{#1}}

\def\R{\mathbb{R}}
\def\calX{\mathcal{X}}
\def\calU{\mathcal{U}}
\def\calZ{\mathcal{Z}}

\def\rmZ{\mathrm{Z}}
\def\Xcstr{\mathrm{X}_{\mathrm{cstr}}}
\def\Xinit{\mathrm{X}_{\mathrm{init}}}
\def\hf{\hat{f}}
\def\xuxp{\sbr{x,u,x^\prime}}

\def\st{\mathrm{ s.t. }}
\def\projX{\mathrm{proj}_\calX}

\newtheorem{definition}{Definition}[section]
\newtheorem{theorem}{Theorem}[section]
\newtheorem{proposition}{Proposition}[section]
\newtheorem{assumption}{Assumption}
\newtheorem{corollary}{Corollary}[theorem]
\newtheorem{lemma}{Lemma}[section]

\newcommand{\change}[1]{#1}

\begin{document}

\title{On the Equilibrium between Feasible Zone and Uncertain Model in Safe Exploration}

\author{Yujie Yang, Zhilong Zheng, Shengbo Eben Li
\thanks{Y. Yang and Z. Zheng contributed equally to this work. All correspondence should be sent to S. Li with email: lishbo@tsinghua.edu.cn.}
\thanks{Yujie Yang, Zhilong Zheng, and Shengbo Eben Li are with School of Vehicle and Mobility and State Key Lab of Intelligent Green Vehicle and Mobility, Tsinghua University, Beijing, 100084, China (e-mail: \{yangyj21,zheng-zl22\}@mails.tsinghua.edu.cn, lishbo@tsinghua.edu.cn).}}



\maketitle

\begin{abstract}
Ensuring the safety of environmental exploration is a critical problem in reinforcement learning (RL).
While limiting exploration to a feasible zone has become widely accepted as a way to ensure safety, key questions remain unresolved: what is the maximum feasible zone achievable through exploration, and how can it be identified?
This paper, for the first time, answers these questions by revealing that the goal of safe exploration is to find the equilibrium between the feasible zone and the environment model.
This conclusion is based on the understanding that these two components are interdependent: a larger feasible zone leads to a more accurate environment model, and a more accurate model, in turn, enables exploring a larger zone.
We propose the first equilibrium-oriented safe exploration framework called safe equilibrium exploration (SEE), which alternates between finding the maximum feasible zone and the least uncertain model.
Using a graph formulation of the uncertain model, we prove that the uncertain model obtained by SEE is monotonically refined, the feasible zones monotonically expand, and both converge to the equilibrium of safe exploration.
Experiments on classic control tasks show that our algorithm successfully expands the feasible zones with zero constraint violation, and achieves the equilibrium of safe exploration within a few iterations.
\end{abstract}


\section{Introduction}
\IEEEPARstart{R}{einforcement} learning (RL) has achieved promising performance in various domains, including video games~\cite{mnih2015human}, board games~\cite{silver2017mastering}, and language modeling~\cite{ouyang2022training}.
However, a critical problem that limits the application of RL in real-world control tasks is its lack of safety in the training process.
This comes from the trial-and-error learning mechanism of RL, where the agent needs to explore the environment and collect interaction data to improve the policy.
In real-world control tasks, unrestricted environmental exploration may lead to undesirable states and violate safety constraints, causing personal injury and property damage.
For example, an autonomous vehicle may collide with surrounding vehicles and humans when exploring a dynamic traffic environment.
This necessitates an exploration mechanism that can guarantee strict constraint satisfaction throughout the training process, also known as safe exploration.

\change{
It is worth noting that the necessity of safe exploration depends heavily on the training mode employed. In scenarios where high-fidelity simulators or massive offline datasets are available, the offline training and online implementation (OTOI) mode is typically adopted \cite{li2023reinforcement}. In this mode, the agent learns in a virtual environment without risking real-world failures, eliminating the need for safe exploration. However, for many real-world applications, such as robotics and autonomous driving, perfect simulators are rarely accessible due to complex physical dynamics and the scarcity of data for specific tasks. This necessitates the simultaneous online training and implementation (SOTI) mode, where the agent learns through online interaction with the real environment \cite{li2023reinforcement}. Whether learning from scratch or fine-tuning a policy pre-trained via OTOI, the agent in the SOTI mode inevitably faces uncertainty outside its known safe region. This requires a mechanism that can simultaneously reduce model uncertainty and expand the safe region while ensuring strict constraint satisfaction.
}

All safe exploration methods share a common underlying mechanism for ensuring safety: restricting environmental exploration in a subset of the state-action space, which we call a \textit{feasible zone}~\cite{li2023reinforcement}.
A feasible zone guarantees that, from any state-action pair within it, there exists a policy that keeps all future states constraint-satisfying.
However, merely identifying a feasible zone is insufficient because it only ensures safety without considering policy performance.
An overly small feasible zone, while ensuring safety, is practically meaningless because the policy cannot accomplish any task within such a limited area.
The size of the feasible zone is crucial in safe exploration since a larger feasible zone allows the policy to operate over a broader region, increasing the likelihood of collecting high-reward data for performance improvement.
Therefore, our objective is to identify the maximum feasible zone.
This problem has been solved in the simulator training setting, where exploration safety is not required~\cite{yang2024synthesizing}.
However, when it comes to safe exploration in real-world environments, finding the maximum feasible zone is far more challenging.
This is because our knowledge of the environment is limited: we can only collect data from within the feasible zone, leaving the environment outside it unexplored and unknown.

A widely used method for safe exploration is the safety filter, which is an action monitoring and intervention module applied after a potentially unsafe policy~\cite{hsu2023safety}.
A safety filter monitors whether the action computed by the policy, given the current state, lies in a feasible zone.
If the action is deemed unsafe, i.e., it results in a deviation from the feasible zone, the safety filter intervenes by replacing the action with one that keeps the resulting state-action pair within the feasible zone.
Most existing safety filter methods rely on human-designed constraints to define feasible zones.
For example, Dalal et al.~\cite{dalal2018safe} compute a minimal correction to a nominal action under an instantaneous state constraint of the environment.
Pham et al.~\cite{pham2018optlayer} augment a neural network policy with a quadratic program (QP) layer that constrains the distances and velocities from the robot to its surrounding objects expressed in linear inequalities.
Cheng et al.~\cite{cheng2019end} use a state-affine control barrier function (CBF) as the constraint of a QP-based safety filter, which is added to a model-free RL policy.
Additionally, several other works also use QP-based safety filters with CBFs in control-affine systems for safe control~\cite{ames2014control, hsu2015control, ames2016control}.
Besides CBFs, the safety index proposed by Liu et al.~\cite{liu2014control} can also serve as a safety filter, which is synthesized as a polynomial function of the state and used for online action projection~\cite{zhao2021model}.
However, a notable limitation of safety filter-based methods is that, to ensure safety, their designed constraints are usually too restrictive, resulting in unreasonably small feasible zones and overly conservative policies.
To address this, Thananjeyan et al.~\cite{thananjeyan2021recovery} propose to use offline data to train a neural network that approximates a feasible zone and is used to switch between a task policy and a recovery policy.
While this approach improves flexibility, its effectiveness is heavily dependent on the coverage of offline data, which is difficult to ensure in safety-critical control tasks where data is often limited.

Another kind of safe exploration method seeks to eliminate the reliance on human-designed constraints by directly learning a safe policy and its corresponding feasible zone from interaction data.
These methods start from a small initial feasible zone and gradually expand it through exploration, while the policy is also improved in this process.
To identify and expand the feasible zone, these methods typically learn an environment model that can extrapolate state transitions or constraint functions from inside the feasible zone to outside its boundaries.
A representative example is the method proposed by Berkenkamp et al.~\cite{berkenkamp2017safe}, which uses a Lyapunov function as a safety certificate for statistical models expressed via a Gaussian process. By collecting environment data, they improve both the dynamic model and control policy, while progressively expanding the safe region.
Gaussian processes are favored in many other works for dynamic model learning and online optimization because of their inherent continuity assumption and data-efficient nature~\cite{sui2015safe, schreiter2015safe, wachi2018safe}.
For example, Turchetta et al.~\cite{turchetta2016safe} address the problem of exploring an environment with an unknown constraint function expressed via a Gaussian process. They collect observations to gain statistical confidence about the safety of unvisited states and safely explore the reachable part of the environment.
In a similar vein, Berkenkamp et al.~\cite{berkenkamp2016safe} focus on learning the region of attraction (ROA) through safe exploration based on a Gaussian process model. They estimate the ROA using a Lyapunov function and expand it by exploring states inside.
Their later work~\cite{berkenkamp2017safe} follows a similar practice by considering safety as stability certified by a Lyapunov function under a Gaussian process model.
In addition to Lyapunov functions, Hamilton-Jacobi (HJ) reachability analysis is another method for learning feasible zones.
Fisac et al.~\cite{fisac2018general} learn a HJ reachability value function with a Gaussian process dynamic model, while Yu et al.~\cite{yu2023safe} extend this method by learning a distributional HJ reachability certificate with an ensemble of Gaussian dynamic models.
Both approaches update their models with interaction data, gradually reducing uncertainty during the exploration process.
As a result, the conservativeness of the HJ reachability constraint diminishes over time.
While these methods can empirically explore larger feasible zones compared to safety filter-based methods, it remains unclear whether they can achieve the maximum feasible zone.
The key takeaway from these methods is that the size of the feasible zone depends on the accuracy of the environment model---a more accurate model corresponds to a larger feasible zone.
This is because determining whether a state-action pair belongs to the feasible zone requires considering the worst-case state transitions predicted by the model.
Although these methods use interaction data to reduce model error, they do not provide a theoretical framework for analyzing the optimal model accuracy achievable through this process.
As a result, the maximum attainable size of the feasible zone remains an open problem.

\begin{figure*}
    \centering
    \includegraphics[width=0.95\linewidth]{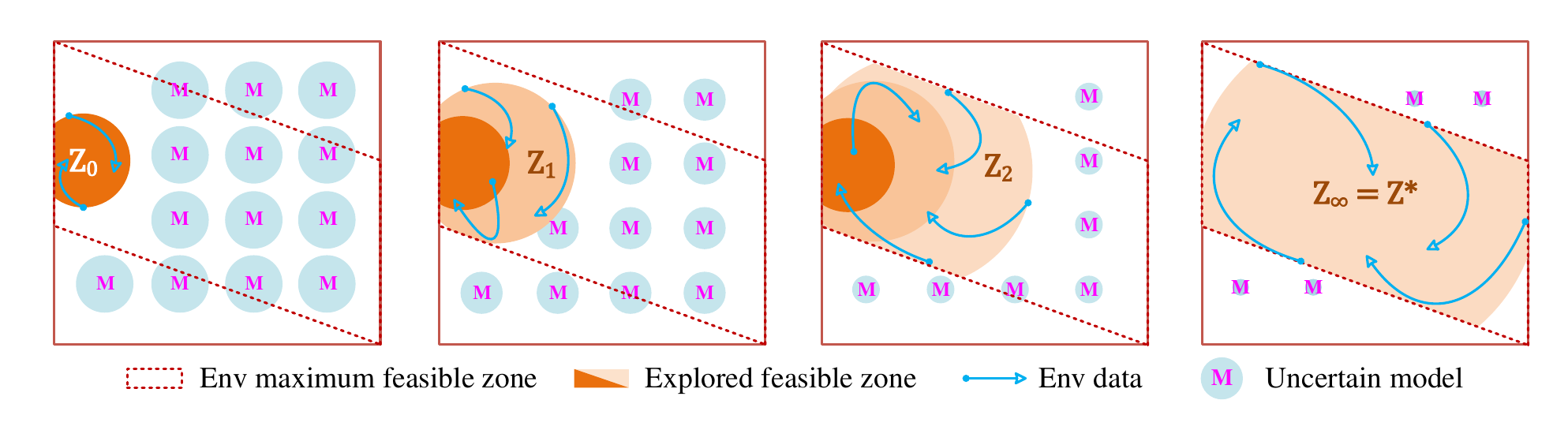}
    \caption{Mechanism of SEE. Exploration starts from a small feasible zone obtained from a prior uncertain model. As more data are collected inside the feasible zone, the model becomes more accurate and the zone is expanded. This process is repeated until the equilibrium between the maximum feasible zone and the least uncertain model is reached.}
    \label{fig: safe exploration mechanism}
\end{figure*}

Through the above analysis, we find that existing methods fail to answer a fundamental question: what is the maximum feasible zone that can theoretically be achieved through safe exploration?
This question is crucial because the maximum feasible zone defines the upper bound on both the safety and performance of any policy learned through safe exploration.
Prior works, particularly methods based on Gaussian processes, suggest that achieving the maximum feasible zone requires finding the most accurate environment model.
However, the feasible zone and the environment model are interdependent in safe exploration.
A more accurate model leads to a larger feasible zone, and a larger feasible zone offers more data over a broader area, which improves model accuracy.
This interdependence implies that the feasible zone and the environment model cannot be optimized separately.
To address this challenge, we propose a safe exploration framework called safe equilibrium exploration (SEE) that simultaneously finds the most accurate environment model and the maximum feasible zone.
Our key insight is that every environment model corresponds to a maximum feasible zone under that model, and every feasible zone also corresponds to an optimal environment model under that zone.
We refer to this optimal model as the \textit{least uncertain model}, which has the minimum error achievable through safe exploration within a given feasible zone.
We point out that the maximum feasible zone in safe exploration is achieved when the \textit{equilibrium} between the feasible zone and environment model is reached, i.e., when they correspond to the maximum feasible zone and the least uncertain model of each other.
Our method iteratively searches for this equilibrium by alternating between computing maximum feasible zones and least uncertain models until convergence, as illustrated in Figure \ref{fig: safe exploration mechanism}.
The main contributions of this paper are summarized as follows:
\begin{enumerate}
\item We propose the notion of the least uncertain model to describe the model with minimum achievable error given a feasible zone.
This model extrapolates state transitions outside the feasible zone with data inside it based on Lipschitz continuity.
Environment data is exploited to the greatest extent so that not a single transition pair can be removed from the least uncertain model.
This notion bridges the environment model and feasible zone, subverting the traditional understanding that the former is learned separately from the latter.
\item We point out that the ultimate goal of safe exploration is to reach the equilibrium between the feasible zone and the uncertain model.
At this equilibrium, the feasible zone cannot be expanded further, and the model uncertainty cannot be reduced any more.
We design an iterative safe exploration framework called SEE to achieve this equilibrium by alternating between finding the maximum feasible zone and the least uncertain model.
We prove that SEE guarantees monotonic expansion of the feasible zone and monotonic decrease of the model uncertainty until the equilibrium is reached.
\item We develop a practical safe exploration algorithm that approximately finds the equilibrium while ensuring strict constraint satisfaction.
The maximum feasible zone is accurately computed by solving a risky Bellman equation using fixed point iteration.
The approximation of our algorithm stems from finding the least uncertain model, which we show to be equivalent to a classic NP-hard problem---the clique decision problem.
We approximately solve this problem in polynomial time based on a sufficient condition for the removability of transition pairs.
Experiments on classic control tasks demonstrate that our algorithm can explore the environment with zero constraint violation and quickly converge to equilibrium.
\end{enumerate}

\section{Preliminaries}
\subsection{Problem description}
We consider a deterministic discrete-time environment
\begin{equation}
    x_{t+1}=f(x_t,u_t),
\end{equation}
where $x\in\calX\subseteq\mathbb{R}^n$ is the state, $u\in\calU\subseteq\mathbb{R}^m$ is the action, and $f:\calX\times\calU\to\calX$ is the unknown environment model. 
The state generated by the mapping of $f$ is called a transition state.
The environment is subject to deterministic state constraints
\begin{equation}
\label{eq: state constraints}
    h(x_t)\le0, t=0,1,2,\dots,\infty,
\end{equation}
where $h:\calX\to\mathbb{R}$ is the constraint function. We define the constrained set as $\mathrm{X_{cstr}}=\{x\in\calX|h(x)\le0\}$.
We aim to find a policy $\pi:\calX\to\calU$ that solves the following constrained optimal control problem:
\begin{equation}
\begin{aligned}
    \max_\pi \ &\mathbb{E}_{x_0\sim\Xinit}\left\{\sum_{t=0}^\infty \gamma^t r(x_t,u_t)\right\} \\
    \mathrm{s.t.} \ &x_{t+1}=f(x_t,u_t) \\
    &h(x_t)\le0, t=0,1,2,\dots,\infty,
\end{aligned}
\end{equation}
where $\Xinit\subseteq\calX$ is the initial state set and $\gamma\in(0,1)$ is the discount factor.
While searching for such a policy, we additionally require that any environmental interaction must satisfy the state constraints \eqref{eq: state constraints}.
This is the essential difference between safe exploration and general safe RL: in safe exploration, not only the final policy is required to be safe, but all intermediate policies must also be safe.
We assume that the environment can be initialized to an arbitrary state, but once the interaction starts, it cannot be terminated until the task is finished.
This means we cannot avoid constraint violations by simply stopping the interaction right before encountering an unsafe state.
Instead, we must ensure the policy is safe in an infinite horizon when exploring the environment.

\subsection{Uncertain model}
If the environment model is completely unknown, accomplishing safe exploration is almost impossible.
This is because, without a model, one can only obtain knowledge of possibly unsafe state transitions through trial and error, which will inevitably lead to constraint violations.
To enable safe exploration, we must have some prior knowledge of the environment, which is described by an uncertain model defined as follows.

\begin{definition}[Uncertain model]
An uncertain model $\hf$ maps a state-action pair to a set of states, i.e., $\hf: \calX\times\calU\to\mathcal{P}(\calX)$, where $\mathcal{P}(\calX)$ is the power set of $\calX$.
\end{definition}

An uncertain model is an inaccurate description of the environment as it maps each state-action pair to a set of states instead of a single state.
The set of states generated by the mapping of an uncertain model is called a transition set.
One way to understand an uncertain model is to think that the prior knowledge of the environment model contains errors.
An uncertain model captures all possible variations of these errors.
For example, if the model error is a random variable with a bounded norm, the transition set would be a ball centered at the transition state of the true model.
The reason for defining an uncertain model with a transition set instead of randomly distributed errors is that we require strict constraint satisfaction. This is possible only when all transition states are safe, even those with low probabilities.
By focusing on the transition set, we explicitly account for all possible transition states, while eliminating the need to consider transition probabilities.
If the uncertain model deviates too far from the true model, it will probably not help safe exploration.
We are interested in the uncertain models that are reasonable estimates of the true model, which is described by the property of well-calibration defined as follows.

\begin{definition}[Well-calibration]
Given an uncertain model $\hf$ and an environment model $f$, $\hf$ is said to be well-calibrated under $f$ if
\begin{equation}
    \forall x\in\calX, u\in\calU, f(x,u)\in\hf(x,u).
\end{equation}
\end{definition}

The aforementioned norm-bounded error is an example of a well-calibrated uncertain model because the true transition state is at the center of, thus included by, the transition set.
With the well-calibration property, we can guarantee safety under the true model by considering all possible cases given by the uncertain model.
If all states in the transition set of an uncertain model are safe, the transition state of the true model must also be safe.

\subsection{Feasible zone}
In safe RL, a feasible region is a set of states that is both constraint-satisfying and control invariant.
In safe exploration, we need to treat feasibility more precisely by additionally taking action into consideration.
Specifically, we need to distinguish the feasibility of actions that can be taken at a state, since there are cases where some actions are acceptable while others are not.
To this end, we extend the concept of a feasible region involving only states to a feasible zone which further includes actions.
We first define the concept of a zone.

\begin{definition}[Zone]
A zone $Z$ is a subset of the Cartesian product of the state space and action space, i.e., $Z\subseteq\calZ=\calX\times\calU$.
\end{definition}

A feasible zone is a set of state-action pairs that are both constraint-satisfying and forward invariant.
Since the true environment model is unknown during exploration, we define a feasible zone under an uncertain model.

\begin{definition}[Feasible zone]
\label{def: feasible zone}
A zone $\rmZ$ is a feasible zone under an uncertain model $\hf$ if
\begin{enumerate}
    \item $\projX(\rmZ)\subseteq\mathrm{X_{cstr}}$,
    \item $\forall(x,u)\in\rmZ,\hf(x,u)\subseteq\projX(\rmZ)$,
\end{enumerate}
where $\projX(\rmZ)=\bbr{x\in\calX|\exists u,\st(x,u)\in\rmZ}$ is the projection of $\rmZ$ on $\calX$.
\end{definition}

The first property in Definition \ref{def: feasible zone} requires that all states in a feasible zone satisfy the state constraint.
The second property requires that any state-action pair in a feasible zone must transfer to a state belonging to some feasible state-action pairs also in this zone so that we can continue to find an available action.
This property ensures that exploration will not leave a feasible zone, which means the feasible zone is forward invariant.
The basic principle of ensuring safe exploration is to restrict environmental interaction within a feasible zone.
In other words, we can only collect interaction data inside a feasible zone.
To learn a high-performance policy, we need to collect data from a large feasible zone---ideally, the maximum feasible zone, which is defined as follows.

\begin{definition}[Maximum feasible zone]
\label{def: maximum feasible zone}
The maximum feasible zone under $\hf$, denoted as $\rmZ^*(\hf)$, is the union of all feasible zones under $\hf$.
\end{definition}

It is important to notice that any feasible zone, including the maximum one, is inherently related to an uncertain model.
This is because the definition of the forward invariance property relies on a transition set, which is directly associated with an uncertain model.
The size of the feasible zone under an uncertain model depends on the sizes of the model's transition sets: smaller transition sets lead to a larger feasible zone because the forward invariance property is more easily satisfied when the transition sets are smaller.
Since the sizes of the transition sets reflect the model's degree of uncertainty, we can conclude that the less uncertain the model is, the larger its feasible zone will be.
Therefore, to maximize the size of an uncertain model's feasible zone, we must reduce its uncertainty to the greatest extent through environmental exploration.

\section{Least Uncertain Model}
To theoretically analyze to what extent a model's uncertainty can be reduced, we propose a concept called the least uncertain model, which exploits all available information from interaction data and continuity of the environment model.

\subsection{Model continuity assumption}
Almost all safe exploration methods assume some kind of continuity on the environment model, e.g., Lipschitz continuity~\cite{berkenkamp2016safe, berkenkamp2017safe}.
However, they only gave intuitive explanations for the necessity of such assumptions without rigorously proving it.
In this section, we prove the necessity of continuity assumptions on the environment model through the mechanism of expanding a feasible zone.
Specifically, we prove that a maximum feasible zone cannot be expanded unless its underlying uncertain model can be updated outside the feasible zone with data inside it, which is only possible under certain continuity assumptions.
First, we define the expandability of a feasible zone.

\begin{definition}[Expandability]
\label{def: expandability}
A feasible zone $\rmZ$ is expandable under an uncertain model $\hf$ if $\rmZ$ is a proper subset of the maximum feasible zone under $\hf$, i.e., $\rmZ\subset\rmZ^*(\hf)$.
\end{definition}

The mechanism of safe exploration is basically a two-step iteration.
The first step is to expand the feasible zone to the maximum one under the current uncertain model.
The second step is to update the uncertain model to reduce its degree of uncertainty, which in turn enlarges its corresponding maximum feasible zone.
This enables the feasible zone to be further expanded in subsequent iterations.
Next, we will discuss the requirements that the updated uncertain model must satisfy to ensure the continuous expansion of the feasible zone.
To begin, we present a necessary and sufficient condition for expandability.

\begin{theorem}
\label{thm: NSC for expandability}
A feasible zone $\rmZ$ is expandable under an uncertain model $\hf$ if and only if there exists $\Delta\rmZ\subseteq\calZ\setminus\rmZ$, such that
\begin{enumerate}
    \item $\projX(\Delta\rmZ)\subseteq\Xcstr$,
    \item $\forall(x,u)\in\Delta\rmZ,\hf(x,u)\subseteq\projX(\rmZ\cup\Delta\rmZ)$.
\end{enumerate}
\end{theorem}
\begin{proof}
First, we prove the necessity. Since $\rmZ$ is expandable under $\hf$, we have $\rmZ\subset\rmZ^*(\hf)$. Take $\Delta\rmZ=\rmZ^*(\hf)\setminus\rmZ$. According to Definition \ref{def: feasible zone}, we have
$$\projX(\Delta\rmZ)\subseteq\projX(\rmZ^*(\hf))\subseteq\Xcstr.$$
Moreover, $\forall(x,u)\in\Delta\rmZ$, it holds that $(x,u)\in\rmZ^*(\hf)$ and therefore we have
$$\hf(x,u)\subseteq\projX(\rmZ^*(\hf))=\projX(\rmZ\cup\Delta\rmZ).$$
Then, we prove the sufficiency. Let $\tilde{\rmZ}=\rmZ\cup\Delta\rmZ$. We have
$$\projX(\tilde{\rmZ})=\projX(\rmZ)\cup\projX(\Delta\rmZ)\subseteq\Xcstr.$$
Moreover, $\forall(x,u)\in\tilde{\rmZ}$, either $(x,u)\in\rmZ$ or $(x,u)\in\Delta\rmZ$ holds. If $(x,u)\in\rmZ$, then $\hf(x,u)\subseteq\projX(\rmZ)$. If $(x,u)\in\Delta\rmZ$, then $\hf(x,u)\subseteq\projX(\rmZ\cup\Delta\rmZ)$. In both cases, we have
$$\hf(x,u)\subseteq\projX(\rmZ\cup\Delta\rmZ)=\projX(\tilde{\rmZ}).$$
Therefore, $\tilde{\rmZ}$ is a feasible zone under $\hf$. According to Definition \ref{def: maximum feasible zone}, we have $\tilde{\rmZ}\subseteq\rmZ^*(\hf)$. Since $\rmZ\subset\tilde{\rmZ}$, we conclude that $\rmZ\subset\rmZ^*(\hf)$, i.e., $\rmZ$ is expandable under $\hf$.
\end{proof}

Theorem \ref{thm: NSC for expandability} reveals that expanding a feasible zone is equivalent to finding a zone such that the union of the two zones is still a feasible zone.
\change{It is worth noting that the maximum feasible zone $\rmZ^*(\hf)$ is not necessarily connected and may consist of multiple disjoint regions. While Theorem \ref{thm: NSC for expandability} theoretically allows the expansion set $\Delta\rmZ$ to be disjoint from the current feasible zone $\rmZ$, practical exploration typically expands $\rmZ$ continuously through its boundary. Therefore, if $\rmZ^*(\hf)$ is disjoint, an exploration process starting from an initial zone $\rmZ_0$ will likely converge to the maximum connected component of $\rmZ^*(\hf)$ that contains $\rmZ_0$. Accessing other disjoint regions would require re-initializing the exploration within those regions.}
Based on Theorem \ref{thm: NSC for expandability}, we can derive a necessary condition for a maximum feasible zone to be expandable under another uncertain model, which is stated as follows.

\begin{corollary}
\label{cor: NC for expandability}
Let $\rmZ$ be the maximum feasible zone under $\hf_1$, i.e., $\rmZ=\rmZ^*(\hf_1)$. If $\rmZ$ is expandable under $\hf_2$, then $\exists(x,u)\in\calZ\setminus\rmZ$ and $x'\in\hf_1(x,u)$, s.t. $x'\notin\hf_2(x,u)$.
\end{corollary}
\begin{proof}
According to Theorem \ref{thm: NSC for expandability}, if $\rmZ$ is expandable under $\hf_2$, then there exists $\Delta\rmZ\subseteq\calZ\setminus\rmZ$, such that
$$\forall(x,u)\in\Delta\rmZ,\hf_2(x,u)\subseteq\projX(\rmZ\cup\Delta\rmZ).$$
Since $\rmZ$ is the maximum feasible zone under $\hf_1$, it is not expandable under $\hf_1$. Thus,
$$\exists(\tilde{x},\tilde{u})\in\Delta\rmZ,\st\hf_1(\tilde{x},\tilde{u})\not\subseteq\projX(\rmZ\cup\Delta\rmZ).$$
In other words,
$$\exists\tilde{x}'\in\hf_1(\tilde{x},\tilde{u}),\st\tilde{x}'\notin\projX(\rmZ\cup\Delta\rmZ).$$
Since $(\tilde{x},\tilde{u})\in\Delta\rmZ$, we have $\hf_2(\tilde{x},\tilde{u})\subseteq\projX(\rmZ\cup\Delta\rmZ)$. Therefore, $\tilde{x}'\notin\hf_2(\tilde{x},\tilde{u})$, which proves the corollary.
\end{proof}

Corollary \ref{cor: NC for expandability} tells us that to expand a maximum feasible zone under an uncertain model, we must have an updated model with a smaller transition set for some state-action pair outside the current feasible zone.
To achieve this, we must reduce the model's uncertainty outside the feasible zone in the second step of safe exploration.
However, interaction data required to update the model can only be collected inside the feasible zone.
In other words, we need to update the certain model outside the feasible zone using only data gathered from within it.
This is possible only if the environment model exhibits some form of continuity, which allows us to extend the information from data inside the feasible zone to areas outside it.
Here, we adopt the widely assumption that the environment model is Lipschitz continuous~\cite{berkenkamp2016safe, berkenkamp2017safe}.

\begin{assumption}[Lipschitz continuity]
The environment model $f$ is $L_f$-Lipschitz continuous, i.e., $\exists L_f\in\mathbb{R}$, s.t. $\forall(x_1,u_1),(x_2,u_2)\in\calZ,$
\change{
$$d_x(f(x_1,u_1),f(x_2,u_2))\le L_fd_{xu}((x_1,u_1),(x_2,u_2)),$$
where $d_x(\cdot,\cdot)$ and $d_{xu}(\cdot,\cdot)$ are distance metrics in the state and state-action spaces, respectively. Generally, both of them are chosen as the Euclidean distance, and they are abbreviated as $d(\cdot,\cdot)$ in the rest of the paper.}
\end{assumption}

The Lipschitz continuity of the environment model is the basis for defining the least uncertain model under a feasible zone, which is introduced in the following section.

\subsection{Least uncertain model}
For a clear and concise definition of the least uncertain model, we first introduce the graph formulation of uncertain models. 
\begin{definition}[Uncertain model graph]
\label{def: uncertain model graph}
Given a constant $L\in\mathbb{R}$, an uncertain model $\hf:\calZ\to\mathcal{P}(\calX)$ can be transformed into a graph, denoted as $D_{\hf}(L)$, where the vertex set $\mathcal{V}_{\hf}=\bbr{\xuxp|(x,u)\in\calZ,x^\prime\in\hf(x,u)}$ and two vertices $V_i,V_j$ are adjacent if $(x_i,u_i)\ne(x_j,u_j)$ and $d\sbr{x_i^\prime, x_j^\prime}\le Ld\sbr{(x_i,u_i),(x_j,u_j)}$.
\end{definition}

Intuitively, each vertex is a possible transition pair given by the uncertain model, and the adjacency between two vertices stands for the satisfaction of $L$-Lipschitz continuity between the two transition pairs.
Note that $L$ only influences the edge set, so in the rest of the paper, when only the vertex set is needed, the constant $L$ will not be specified.

A model's degree of uncertainty is reflected by the cardinality of its vertex set $|\mathcal{V}_{\hf}|=\sum_i|\hf(x_i,u_i)|$.
Updating an uncertain model is essentially removing as many vertices as possible from its vertex set.
There are two kinds of vertices that can be removed.
The first kind consists of vertices inside the feasible zone that do not align with the true environment model.
Since exploration within the feasible zone is always safe, we can fully explore this area and collect all transition pairs corresponding to the true environment model.
With these transition pairs, we can update all transition sets inside the feasible zone to single-element sets that contain only the true transition state.
Therefore, all other states in the transition set, except for the true transition state, can be removed.
This property is formally defined as follows. 

\begin{definition}[Removability of the first kind]
\label{def: Removability of the first kind}
Given a feasible zone $\rmZ$ and a well-calibrated uncertain model $\hf$, a vertex $\xuxp$ in $\mathcal{V}_{\hf}$ is said to be removable of the first kind under $\rmZ$ if $(x,u)\in\rmZ$ and $x^\prime\neq f(x,u)$, where $f$ is the true environment model.
\end{definition}

The second kind of removable vertices are those that violate the Lipschitz continuity.
To understand this, let's first consider what a possible candidate for the true environment model should look like, given a well-calibrated uncertain model.
A candidate model should be Lipschitz continuous and, of course, it should be contained within the uncertain model.
Any model that satisfies these two properties can potentially be the true environment model, and therefore, all its transition pairs should be kept.
On the contrary, if a transition pair does not belong to any of these candidate models, it can be considered invalid and therefore, removed.
This property is formally defined as follows.

\begin{definition}[Removability of the second kind]
\label{def: Removability of the second kind}
Given a constant $L$ and a well-calibrated uncertain model $\hf$, a vertex $\xuxp$ of $D_{\hf}(L)$ is said to be removable of the second kind under $L$ if there does not exist $f_0:\calZ\to\calX$ such that
\begin{enumerate}
    \item $x^\prime=f_0(x,u)$,
    \item $f_0$ is $L$-Lipschitz continuous,
    \item $\hf$ is well-calibrated under $f_0$.
\end{enumerate}
\end{definition}

Note that in both Definition \ref{def: uncertain model graph} and \ref{def: Removability of the second kind}, we use an arbitrary constant $L$ instead of the Lipschitz constant of the true environment model $L_f$.
This is because we want to make these definitions applicable even in cases where $L_f$ is unknown. 
A proper choice of $L$ should be $L\ge L_f$ because, in this case, we conservatively regard more transition pairs as valid, ensuring that the transition pairs of the true environment model are kept.
Removability of the second kind can be equivalently expressed in more concise graph-theoretic language as follows.

\begin{theorem}
\label{thm: removability of the second kind}
Given a constant $L$ and a well-calibrated uncertain model $\hf$, a vertex $\xuxp$ of $D_{\hf}(L)$ is removable of the second kind if and only if $\xuxp$ does not belong to any $N$-clique (i.e., $N$-complete subgraph) of $D_{\hf}(L)$, where $N=|\calZ|$.
\end{theorem}
\begin{proof}
First, we prove the necessity.
Suppose that $\xuxp$ is removable of the second kind and belongs to an $N$-clique of $D_{\hf}(L)$.
Since any two vertices of an $N$-clique are adjacent, and vertices with the same $(x,u)$ are not adjacent, the $N$-clique must contain all $(x,u)$ in $\calZ$.
Thus, we can construct a mapping $f_0$ using this $N$-clique, and it satisfies $x^\prime=f_0(x,u)$.
It holds that $\hf$ is well-calibrated under $f_0$ because $f_0$ comes from a subgraph of $D_{\hf}(L)$.
It also holds that $f_0$ is $L$-Lipschitz continuous because all vertices are adjacent.
According to Definition \ref{def: Removability of the second kind}, $\xuxp$ is not removable of the second kind, which contradicts our assumption.
Thus, $\xuxp$ does not belong to any $N$-clique of $D_{\hf}(L)$.

Next, we prove the sufficiency.
Suppose $\xuxp$ does not belong to any $N$-clique of $D_{\hf}(L)$ and is not removable of the second kind, i.e., there exists an $f_0$ that satisfies all three conditions in Definition \ref{def: Removability of the second kind}.
Consider a dummy uncertain model $\hf_0$ obtained by recasting the output of $f_0$ at any point into a single-element set.
It holds that $\hf_0$ is an $N$-clique of $D_{\hf}(L)$ and $\xuxp$ belongs to this $N$-clique, which contradicts our assumption.
Thus, $\xuxp$ is removable of the second kind.
\end{proof}

By removing vertices of both the first and second kinds, we obtain the least uncertain model, which is defined as follows.
Note that for simplicity, we define the following notations for uncertain models:
$$
\begin{aligned}
\hf_1=\hf_2\iff\forall (x,u)\in\mathcal{Z},\hf_1(x,u)=\hf_2(x,u), \\
\hf_1\subseteq\hf_2\iff\forall (x,u)\in\mathcal{Z},\hf_1(x,u)\subseteq\hf_2(x,u).
\end{aligned}
$$

\begin{definition}[Least uncertain model]
\label{def: least uncertain model}
Given a feasible zone $\rmZ$, a well-calibrated uncertain model $\hf$ and a constant $L$, $\hf'$ is said to be the $L$-Lipschitz least uncertain model under $\rmZ$ and $\hf$, denoted as $\hf'=\hf^*(\rmZ,\hf;L)$, if it satisfies the following three properties:
\begin{enumerate}
    \item \label{least uncertain model cond1} (Containment) $\hf' \subseteq \hf$, 
    \item \label{least uncertain model cond2} (Imprunability) Every vertex of $\hf'$ is neither removable of the first kind under $\rmZ$ nor removable of the second kind under $L$,
    \item (Maximality) $\forall \hf'' \text{ satisfies \ref{least uncertain model cond1} and \ref{least uncertain model cond2} }, \hf''\subseteq\hf'$.
\end{enumerate}
\end{definition}

The above definition ensures that the least uncertain model contains as few transition pairs as possible while preserving the well-calibration property.
In this sense, the least uncertain model is the best model we can obtain given the knowledge of a feasible zone $\rmZ$ and a prior uncertain model $\hf_0$.

\section{Safe Equilibrium Exploration}
\subsection{Equilibrium of safe exploration}
We first briefly introduce the mechanism of safe exploration, which is illustrated in Figure \ref{fig: safe exploration mechanism}.
At the initial stage, before any data is collected, our knowledge of the environment is only a prior uncertain model.
The size of its transition set reflects the accuracy of our prior knowledge.
In most cases, prior knowledge is highly inaccurate, leading to large transition sets and consequently, a small maximum feasible zone.
Within this feasible zone, we can fully explore the environment and collect transition data on all state-action pairs.
With these data, we can update the uncertain model to reduce the sizes of its transition sets.
This leads to a larger maximum feasible zone in which we can collect more transition data.
By iterating this process, the uncertain model becomes more and more accurate and the feasible zone becomes larger and larger.
Finally, this process reaches an equilibrium where neither the uncertain model nor the feasible zone can be further updated, i.e., the model is the least uncertain model under the feasible zone, and the zone is the maximum feasible zone under the uncertain model.
This equilibrium is the final goal of safe exploration, which is formally defined as follows.

\begin{definition}
\label{def: equilibrium}
The equilibrium of safe exploration is a pair of a feasible zone and an uncertain model $(\rmZ,\hf)$ that satisfies $\rmZ=\rmZ^*(\hf)$ and $\hf=\hf^*(\rmZ,\hf_0)$. At this point, $\rmZ$ is called the explorable maximum feasible zone, and $\hf$ is called the explorable least uncertain model.
\end{definition}

It is worth noticing that the explorable feasible zone may not equal the maximum feasible zone under the true environment model.
This is because, at the equilibrium, the uncertain model may still have some degree of uncertainty.
The state-action pairs outside the explorable feasible zone can no longer be safely explored unless further continuity properties of the environment model are provided.

\subsection{Theoretical analysis of safe exploration}

\begin{algorithm}
\caption{Safe equilibrium exploration (SEE)}
\label{alg: SEE}
\KwIn{initial uncertain model $\hf_0$, constant $L$.}
\For{each iteration $k$}{
    Find maximum feasible zone: $\rmZ_k=\rmZ^*(\hf_{k-1})$\;
    Find least uncertain model: $\hf_k=\hf^*(\rmZ_k,\hf_{k-1};L)$.
}
\end{algorithm}

Our safe exploration algorithm, which iteratively performs the above two steps, is called safe equilibrium exploration (SEE), as outlined in Algorithm \ref{alg: SEE}.
\change{It is worth noting that the least uncertain model operator $\hf^*(\rmZ, \hf; L)$ defined in Definition \ref{def: least uncertain model} is applicable to any pair of set $\rmZ$ and model $\hf$. The set $\rmZ$ is not required to be a feasible zone of $\hf$. In Algorithm \ref{alg: SEE}, we apply this operator recursively using $\hf_{k-1}$ and $\rmZ_k$, but the theoretical properties of the operator hold generally.}
In the rest of this section, we prove several important properties of SEE, including the monotonicity of the uncertain model and feasible zone, and the convergence of SEE.
First, we show that the model's degree of uncertainty is reduced in each iteration.
This property is called model refinement, which is formally stated in the following theorem.

\begin{theorem}[Model refinement]
\label{thm: model refinement}
The uncertain models obtained by Algorithm $\ref{alg: SEE}$ satisfy $\hf_{k+1}\subseteq\hf_k,\forall k\in\mathbb{N}$.
\end{theorem}
\begin{proof}
Since $\hf_k=\hf^*(\rmZ_k,\hf_{k-1};L)$, according to the containment of the least uncertain model, we have $\hf_{k+1}\subseteq\hf_k$.
\end{proof}

As we expected, a less uncertain model leads to a larger feasible zone.
This property is called zone expansion, which is formally stated in the following theorem.

\begin{theorem}[Zone expansion]
The feasible zones obtained by Algorithm $\ref{alg: SEE}$ satisfy $\rmZ_k\subseteq\rmZ_{k+1},\forall k\in\mathbb{N}$.
\end{theorem}
\begin{proof}
According to Definition \ref{def: feasible zone},
$$\forall(x,u)\in\rmZ_k,\hf_k(x,u)\subseteq\projX(\rmZ_k).$$
According to Theorem \ref{thm: model refinement},
$$\forall(x,u)\in\rmZ_k,\hf_{k+1}(x,u)\subseteq\hf_k(x,u)\subseteq\projX(\rmZ_k).$$
Therefore, $\rmZ_k$ is a feasible zone under $\hf_{k+1}$. Since $\rmZ_{k+1}$ is the maximum feasible zone under $\hf_{k+1}$, we have $\rmZ_k\subseteq\rmZ_{k+1},\forall k\in\mathbb{N}$.
\end{proof}

Zone expansion is proved by showing that a feasible zone under an uncertain model is still feasible under another less uncertain model.
Next, we set out to prove the convergence of SEE.
Before that, we need to prove some containment relationships of the least uncertain models, which are stated in the following lemmas. 

\begin{lemma}
\label{lem: uncertain model containment 1}
For an arbitrary feasible zone $\rmZ$ and well-calibrated uncertain models $\hf_1$, $\hf_2$, we have
$$\hf^*(\rmZ, \hf_1;L)\subseteq\hf_2 \implies \hf^*(\rmZ, \hf_1;L)\subseteq\hf^*(\rmZ, \hf_2;L).$$
\end{lemma}

\begin{proof}
According to Definition \ref{def: least uncertain model}, $\hf^*(\rmZ, \hf_1;L)$ is imprunable under $\rmZ$ and $L$.
Together with $\hf^*(\rmZ, \hf_1;L)\subseteq\hf_2$, according to the maximality of the least uncertain model, we have $\hf^*(\rmZ, \hf_1;L)\subseteq\hf^*(\rmZ, \hf_2;L)$.
\end{proof}

The above lemma tells us that we cannot obtain more information from a model that is more uncertain than a currently available least uncertain model.

\begin{lemma}
\label{lem: uncertain model containment 2}
For arbitrary feasible zones $\rmZ_1$, $\rmZ_2$ and a well-calibrated uncertain model $\hf$, we have
$$\rmZ_2 \subseteq \rmZ_1 \implies \hf^*(\rmZ_1, \hf;L) \subseteq \hf^*(\rmZ_2, \hf;L).$$
\end{lemma}

\begin{proof}
It holds by definition that $\hf^*(\rmZ_1, \hf;L) \subseteq \hf$. Since $\hf^*(\rmZ_1, \hf;L)$ is imprunable under $\rmZ_1$, it must also be imprunable under $\rmZ_2$. Thus, we have $\hf^*(\rmZ_1, \hf_1;L) \subseteq \hf^*(\rmZ_2, \hf_1;L)$.
\end{proof}

The above lemma tells us that a larger feasible zone leads to a less uncertain model.
With the above two lemmas, we can prove an equivalent form of the least uncertain models obtained by SEE, which serves as a core element in the convergence proof.

\begin{proposition}[Equivalent form of least uncertain model]
\label{pro: equivalent form}
The uncertain models obtained by Algorithm $\ref{alg: SEE}$ satisfy $\hf_k=\hf^*(\rmZ_k,\hf_0;L),\forall k\in\mathbb{N}^+$.
\end{proposition}
\begin{proof}
First, we prove that $\hf_k\subseteq\hf^*(\rmZ_k,\hf_0;L)$.
According to Theorem \ref{thm: model refinement}, we have $\hf_k\subseteq\hf_0$.
Since $\hf_k=\hf^*(\rmZ_k,\hf_{k-1};L)$, according to Lemma \ref{lem: uncertain model containment 1}, we have
$$\hf_k=\hf^*(\rmZ_k,\hf_{k-1};L)\subseteq\hf^*(\rmZ_k,\hf_0;L).$$

Next, we prove $\hf^*(\rmZ_k,\hf_0;L)\subseteq\hf_{k}$.
According to Lemma \ref{lem: uncertain model containment 1}, if we can prove $\hf^*(\rmZ_k,\hf_0;L)\subseteq\hf_{k-1}$, then we have $$\hf^*(\rmZ_k,\hf_0;L)\subseteq\hf^*(\rmZ_k,\hf_{k-1};L)=\hf_{k}.$$
We prove $\hf^*(\rmZ_k,\hf_0;L)\subseteq\hf_{k-1},\forall k\in\mathbb{N}^+$ by mathematical induction.
When $k=1$, $\hf^*(\rmZ_1,\hf_0;L)\subseteq\hf_0$ holds by definition.
Assume that when $k=n,n\in\mathbb{N}^+$, $\hf^*(\rmZ_n,\hf_0;L)\subseteq\hf_{n-1}$. 
Then, we have
\begin{flalign*}
&&\hf^*(\rmZ_{n+1},\hf_0;L)&\subseteq\hf^*(\rmZ_n,\hf_0;L)&\text{(Lemma \ref{lem: uncertain model containment 2})}
\\
&& &\subseteq\hf^*(\rmZ_n,\hf_{n-1};L)  &\text{(Lemma \ref{lem: uncertain model containment 1})}
\\
&& &=\hf_n.&
\end{flalign*}
Thus, we have $\hf^*(\rmZ_k,\hf_0;L)\subseteq\hf_{k-1},\forall k\in\mathbb{N}^+$, and consequently, $\hf^*(\rmZ_k,\hf_0;L)\subseteq\hf_k,\forall k\in\mathbb{N}^+$.

Therefore, we can conclude that $\hf_k=\hf^*(\rmZ_k,\hf_0;L),\forall k\in\mathbb{N}^+$.
\end{proof}

With the above proposition, we can prove the convergence of SEE.

\begin{theorem}[Convergence of SEE]
In Algorithm $\ref{alg: SEE}$, if $\rmZ_{k+1}=\rmZ_k$ and $\hf_{k+1}=\hf_k$, then $(\rmZ_k,\hf_k)$ is the equilibrium of safe exploration.
\end{theorem}
\begin{proof}
Since $\rmZ_{k+1}=\rmZ_k$, we have
$$\rmZ_k=\rmZ_{k+1}=\rmZ^*(\hf_k).$$
Since $\hf_{k+1}=\hf_k$ and using Proposition \ref{pro: equivalent form}, we have
$$\hf_k=\hf_{k+1}=\hf^*(\rmZ_{k+1},\hf_0;L)=\hf^*(\rmZ_k,\hf_0;L).$$
Therefore, according to Definition \ref{def: equilibrium}, $(\rmZ_k,\hf_k)$ is the equilibrium of safe exploration.
\end{proof}

The above theorem tells us that when both the feasible zone and uncertain model are not changed in an iteration, the equilibrium of safe exploration is reached.
Together with the monotonicity of the feasible zone and uncertain model, \change{we can conclude that the exploration converge to the equilibrium. When the state and action spaces are finite, the convergence is reached in a finite number of iterations. This is because, in finite spaces, the numbers of feasible zones and uncertain models are both finite, the equilibrium must be reached before traversing all of them.}
In the following two subsections, we introduce how to find the maximum feasible zone and the least uncertain model in each iteration of SEE.

\subsection{Finding maximum feasible zone}
We represent a feasible zone by a function called constraint decay function~\cite{yang2024synthesizing}, which is defined as follows.

\begin{definition}[Constraint decay function]
The constraint decay function (CDF) of a policy $\pi$ under an uncertain model $\hf$, denoted as $G^\pi_{\hf}$, is defined as
\begin{equation}
    G^\pi_{\hf}(x,u)\coloneqq\gamma^{N^\pi_{\hf}(x,u)},\forall(x,u)\in\calZ,
\end{equation}
where $\gamma\in(0,1)$ is the discount factor and
$$
\begin{aligned}
N^\pi_{\hf}(x,u)\coloneqq&\min_{t,x_1,x_2,\dots}t,\\    
\st\ &h(x_t)>0,t\in\mathbb{N},\\
&x_0=x,x_1\in\hf(x,u),\\
&x_{i+1}\in\hf(x_i,\pi(x_i)),i\in\mathbb{N}^+.
\end{aligned}
$$
\end{definition}

In the above definition, $N^\pi_{\hf}(x,u)$ is the minimum number of time steps to constraint violation starting from $(x,u)$ under policy $\pi$ and the worst-case state transition supported by $\hf$.
If $N^\pi_{\hf}(x,u)=+\infty$, it means that the constraint will never violated, and $(x,u)$ is thus in the feasible zone.
Otherwise, if $N^\pi_{\hf}(x,u)<+\infty$, the constraint will be violated at some time in the future, meaning that $(x,u)$ is not feasible.
Therefore, the zero-sublevel set of $G^\pi_{\hf}(x,u)$, i.e., $\{(x,u)\in\calZ|G^\pi_{\hf}(x,u)\le0\}$, is the feasible zone under policy $\pi$.

The value of the CDF indicates the urgency associated with a constraint violation. Higher values signify greater urgency and, consequently, lower levels of safety. To prioritize safety, it is desirable to minimize the CDF value. This consideration naturally gives rise to the concept of the optimal CDF.

\begin{definition}[Optimal CDF]
The optimal CDF under $\hf$, denoted as $G^*_{\hf}$, is defined as
\begin{equation}
    G^*_{\hf}(x,u)\coloneqq\inf_\pi G^\pi_{\hf}(x,u),\forall(x,u)\in\calZ.
\end{equation}
\end{definition}

The primary reason we are interested in the optimal CDF is that its zero-level set represents the maximum feasible zone, as stated in the following theorem.

\begin{theorem}
The zero-level set of the optimal CDF under $\hf$ is the maximum feasible zone under $\hf$, i.e.,
\begin{equation}
    Z_{G^*_{\hf}}\coloneqq\{(x,u)\in\calZ|G^*_{\hf}(x,u)=0\}=\rmZ^*(\hf).
\end{equation}
\end{theorem}
\begin{proof}
First, we prove that $Z_{G^*_{\hf}}$ is a feasible zone. We have
$$\forall(x,u)\in Z_{G^*_{\hf}},G^*_{\hf}(x,u)=0\Rightarrow N^*_{\hf}(x,u)=\infty\Rightarrow h(x)\le0.$$
Thus, $Z_{G^*_{\hf}}\subseteq\Xcstr$. Assume that
$$\exists(x,u)\in Z_{G^*_{\hf}},\st\hf(x,u)\not\subseteq\projX(Z_{G^*_{\hf}}),$$
i.e.,
$$\exists x'\in\hf(x,u),\st\min_{u'\in\calU}G^*_{\hf}(x',u')>0.$$
Then,
$$N^*_{\hf}(x,u)\le1+\min_{u'\in\calU}N^*_{\hf}(x',u')<\infty,$$
which is contradictory with $(x,u)\in Z_{G^*_{\hf}}$. Thus,
$$\forall(x,u)\in Z_{G^*_{\hf}},\hf(x,u)\subseteq\projX(Z_{G^*_{\hf}}).$$
Therefore, according to Definition \ref{def: feasible zone}, $Z_{G^*_{\hf}}$ is a feasible zone.

Next, we prove that $\rmZ^*(\hf)\in Z_{G^*_{\hf}}$. Assume that
$$\exists(x,u)\in\rmZ^*(\hf),\st G^*_{\hf}(x,u)\ne0.$$
This means that
$$
\begin{aligned}
&\forall u_1,u_2,u_3,\dots\in\calU,\\
&\exists x_1\in\hf(x,u),x_2\in\hf(x_1,u_1),\dots,\mathrm{and}\ T<\infty,\\
&\st\ h(x_T)>0.
\end{aligned}
$$
Thus, $x_T\notin\projX(\rmZ^*(\hf))$. This is contradictory with the fact that $\rmZ^*(\hf)$ is control invariant. Therefore,
$$\forall(x,u)\in\rmZ^*(\hf),G^*_{\hf}(x,u)=0.$$
Hence, we conclude that $Z_{G^*_{\hf}}=\rmZ^*(\hf)$.
\end{proof}

This theorem tells us that we can find the maximum feasible zone by computing the optimal CDF.
The next theorem further provides a method for computing the optimal CDF, which is also a necessary and sufficient condition for optimal CDF.

\begin{theorem}[Risky Bellman equation]
$G_{\hf}$ is the optimal CDF under $\hf$ if and only if it satisfies the risky Bellman equation:
\begin{equation}
\label{equ: risky Bellman equation}
\begin{aligned}
G_{\hf}(x,u)=c(x)+(1-c(x))\gamma\max_{x'\in\hf(x,u)}\min_{u'\in\calU}G_{\hf}(x',u'),\\
\forall(x,u)\in\calZ,
\end{aligned}
\end{equation}
where $c(x)=[h(x)>0]$, i.e., $c(x)=1$ if $h(x)>0$ else $0$.
\end{theorem}
\begin{proof}
First, we prove that $G^*_{\hf}$ satisfies \eqref{equ: risky Bellman equation}. $\forall(x,u)\in\calZ$, if $c(x)=1$, then $G^*_{\hf}=1=c(x)$ and the equality in \eqref{equ: risky Bellman equation} holds. If $c(x)=0$, then
$$
\begin{aligned}
N^*_{\hf}(x,u)=&\max_\pi\min_{t,x_1,x_2,\dots}t\\
=&1+\max_\pi\min_{t,x_1,x_2,\dots}t-1\\
=&1+\min_{x_1}\max_{u_1\in\calU}\max_\pi\min_{t,x_2,x_3,\dots}t-1,\\
\st\ &h(x_t)>0,t\in\mathbb{N},\\
&x_1\in\hf(x,u),x_2\in\hf(x_1,u_1),\\
&x_{i+2}\in\hf(x_{i+1},\pi(x_{i+1})),i\in\mathbb{N}^+,
\end{aligned}
$$
Thus,
$$N^*_{\hf}(x,u)=1+\min_{x'\in\hf(x,u)}\max_{u'\in\calU}N^*_{\hf}(x',u'),$$
which means
$$G^*_{\hf}(x,u)=\gamma\max_{x'\in\hf(x,u)}\min_{u'\in\calU}G^*_{\hf}(x',u'),$$
and the equality in \eqref{equ: risky Bellman equation} holds. Therefore, $G^*_{\hf}$ satisfies \eqref{equ: risky Bellman equation}.

Next, we prove that the risky Bellman operator
$$(BG_{\hf})(x,u)\coloneqq c(x)+(1-c(x))\gamma\max_{x'\in\hf(x,u)}\min_{u'\in\calU}G_{\hf}(x',u')$$
is a contraction mapping under the uniform norm. $\forall G_{\hf},\tilde{G}_{\hf}$, we have
$$
\begin{aligned}
(BG_{\hf})(x,u)-(B\tilde{G}_{\hf})(x,u)=(1-c(x))\gamma\Big(\\
\max_{x'\in\hf(x,u)}\min_{u'\in\calU}G_{\hf}(x',u')-\max_{x'\in\hf(x,u)}\min_{u'\in\calU}\tilde{G}_{\hf}(x',u')\Big).
\end{aligned}
$$
Using the relationship
$$
\begin{aligned}    
&\left|\max_a\min_b f_1(a,b)-\max_a\min_b f_2(a,b)\right|\\
\le&\max_a\max_b\left|f_1(a,b)-f_2(a,b)\right|,
\end{aligned}
$$
we have
$$
\begin{aligned}    
&\left\Vert BG_{\hf}-B\tilde{G}_{\hf}\right\Vert_\infty\\
\le&\gamma\max_{x'\in\hf(x,u)}\max_{u'\in\calU}\left|G_{\hf}(x',u')-\tilde{G}_{\hf}(x',u')\right|\\
=&\gamma\left\Vert G_{\hf}-\tilde{G}_{\hf}\right\Vert_\infty.
\end{aligned}
$$
According to Banach’s fixed-point theorem, the risky Bellman operator $B$ has a unique fixed point. This means that the solution to the risky Bellman equation \eqref{equ: risky Bellman equation} is also unique, and this solution is exactly $G^*_{\hf}$.
\end{proof}

This theorem enables us to compute the optimal CDF through fixed point iteration, which involves repeatedly applying the right-hand side of the risky Bellman equation starting from an arbitrary initial function.
Once the optimal CDF is computed, the maximum feasible zone can be identified as its zero-sublevel set.
This algorithm for finding the maximum feasible zone is called feasible zone iteration, as detailed in Algorithm \ref{alg: feasible zone iteration}.

\begin{algorithm}
\caption{Feasible zone iteration}
\label{alg: feasible zone iteration}
\KwIn{initial constraint decay function $G_0$, uncertain model $\hf$.}
\For{each iteration $k$}{
    \For{each $(x,u)\in\calZ$}{
        $G_{k+1}(x,u)=c(x)+(1-c(x))\gamma\max_{x'\in\hf(x,u)}\min_{u'\in\mathcal{U}}G_k(x',u').$
    }
}
\end{algorithm}

\subsection{Finding least uncertain model}
Starting with a previous uncertain model $\hf_{k-1}$ and a newly-expanded feasible zone $\rmZ_k$, we find the least uncertain model $\hf_k=\hf^*(\rmZ_k,\hf_{k-1})$ by first identifying and extracting removable vertices of the first kind and then do the same thing for the second kind until there is no removable vertices. It is obvious that the resulting uncertain model has containment and imprunability defined in Definition \ref{def: least uncertain model}. Additionally, throughout the process of pruning, a vertex that is previously considered as removable will never become unremovable. Take removability of the second kind as an instance, a vertex that does not belong to any $N$-clique of $D_{\hf}(L)$ at the very beginning will never do afterwards. Hence, we can never mistakenly delete a vertex, so the resulting uncertain model also has maximality, which makes it the $L$-Lipschitz least uncertain model under $\rmZ_k$ and $\hf_{k-1}$.

For removable vertices of the first kind, we can easily extract them by traversing all vertices whose state-action pairs are in $\rmZ$ and leaving only those satisfying $x^\prime=f(x,u)$. This is equivalent to set $\hf(x,u)=\bbr{f(x,u)}$ for all $(x,u)$ in $\rmZ$.
However, to identify removable vertices of the second kind, according to Theorem \ref{thm: removability of the second kind}, it is required to check whether a vertex belongs to any $N$-clique. In other words, we need to solve the following removability decision problem.
\begin{definition}[Removability decision problem]
In a removability (of the second kind) decision problem (shortened as RDP), the input is an uncertain model graph $D_{\hf}(L)$ and a vertex $\xuxp$ of it, and the output is a boolean value that equals true if $\xuxp$ belongs to any $N$-clique, and false otherwise, where $N=|\calZ|$.
\end{definition}

If we can solve the RDP for $D_{\hf}(L)$ and every vertex of it, we can also solve the clique decision problem for $D_{\hf}(L)$ and $N$, which asks whether $D_{\hf}(L)$ has a $N$-clique. The clique decision problem is one of Karp's 21 NP-complete problems (also NP-hard)~\cite{karp2010reducibility}. To the best of our knowledge, even though an uncertain model graph is by definition an $N$-partite graph (a graph whose vertices set can be partitioned into $N$ different independent sets), which facilitates the searching for $N$-cliques, there is still no polynomial-time algorithm for finding $N$-cliques in $N$-partite graph so far.


Since there is no efficient way to discard all removable vertex of the second kind through the equivalent condition in Theorem \ref{thm: removability of the second kind}, we resort to a sufficient condition given by the following theorem.

\begin{theorem}
\label{thm: removability of the second kind sufficient cond}
Given a constant $L$ and a well-calibrated uncertain model $\hf$, a vertex $\xuxp$ of $D_{\hf}(L)$ is removable of the second kind under $L$ if $\xuxp$ does not have neighbors of all colors different from its own.
\end{theorem}

\begin{proof}
Let $\xuxp$ be a vertex that does not have neighbors of all colors different from its own. Note from the definition of uncertain model graph that any two vertices of the same color must not be adjacent. Since there are $N$ colors in total, it follows that any $N$-clique must contain exactly one vertex of each color. Thus, $\xuxp$ does not belong to any $N$-clique and hence is removable of the second kind.
\end{proof}

Guided by Theorem \ref{thm: removability of the second kind sufficient cond}, we propose an approximate algorithm for solving the least uncertain model called uncertain model iteration, as detailed in Algorithm \ref{alg: uncertain model iteration}.
The majority of uncertain model iteration is four nested loops. Supposing that the size of every transition set is bounded by $M$, the while-loop will repeat at most $NM$ times, and the three for-loops will repeat at most $N$, $M$, $N$ times, from outside to inside respectively. Additionally, the calculation of the distance from a point to a transition set in the innermost loop has a time complexity of $O(M)$. Thus, the time complexity of Algorithm \ref{alg: uncertain model iteration} is $O(N^3M^3)$ in total. Using this approximate algorithm for the step of finding the least uncertain model in Algorithm \ref{alg: SEE}, we obtain a practical algorithm for safe exploration. This algorithm preserves the monotonic refinement of the uncertain model and the monotonic expansion of the feasible zone because it is still removing vertices from the model graph, leading to potentially more feasible state-action pairs.

\begin{algorithm}
\caption{Uncertain model iteration}
\label{alg: uncertain model iteration}
\KwIn{initial uncertain model $\hf_0$, feasible zone $\rmZ$, constant $L$.}
$\hf\leftarrow\hf_0$. \\
\For{each $(x,u)\in\rmZ$}{
    $\hf(x,u)=\{f(x,u)\}$.
}
\While{true}{
    $T_\mathrm{rm}=\emptyset$\;
    \For{each $(x_1,u_1)\in\calZ\setminus\rmZ$}{
        \For{each $x_1'\in\hf(x_1,u_1)$}{
            \For{each $(x_2,u_2)\in\calZ$}{
                $d=d(x_1',\hf(x_2,u_2))$\;
                \If{$d>Ld((x_1,u_1),(x_2,u_2))$}{
                    $T_\mathrm{rm}=T_\mathrm{rm}\cup\{(x_1,u_1,x_1')\}$\;
                    \textbf{break}
                }
            }
        }
    }
    \If{$T_\mathrm{rm}=\emptyset$}{
        \textbf{break}
    }
    \For{each $(x,u,x')\in T_\mathrm{rm}$}{
        $\hf(x,u)=\hf(x,u)\setminus\{x'\}$.
    }
}
\end{algorithm}

\section{Experiments}
In this section, we test our safe exploration algorithm on \change{three} classic control tasks: \change{a 2D linear double integrator, a 2D nonlinear pendulum, and a 3D nonlinear unicycle}. For each task, we visualize the feasible zones and uncertain models during exploration. We also compare the final feasible zones with the theoretical maximum feasible zones under the true models \change{for two 2D tasks}.

\subsection{Double integrator}
Consider a double integrator with the following dynamics:
\begin{equation}
\begin{aligned}
    x_{t+1}(1)&=x_t(1)+x_t(2), \\
    x_{t+1}(2)&=x_t(2)+u_t,
\end{aligned}
\end{equation}
where $x(i)$ denotes the $i$-th element of $x$.
The state space and action space are discrete and only take integers:
\begin{equation}
\label{eq: double integrator state and action spaces}
\begin{aligned}
    \calX&=\mathbb{Z}^2\cap[-20,20]\times[-15,15], \\
    \calU&=\mathbb{Z}\cap[-2,2].
\end{aligned}
\end{equation}
The constrained set equals the state space:
\begin{equation}
    \Xcstr=\calX.
\end{equation}
This setting may be a bit confusing since all states in the state space are constraint-satisfying. Strictly speaking, the state space is larger than $\calX$ in \eqref{eq: double integrator state and action spaces} and can be considered as the whole $\mathbb{Z}^2$. However, we can immediately terminate the task once the state goes out of $\calX$ so that the existence of states outside $\calX$ is only for indication of constraint violation. Since our primary focus is on exploring the feasible zone and these states can't be feasible, we exclude them from the state space for simplicity.

The initial uncertain model is chosen so that every transition set is a ball centered at the true transition state:
\begin{equation}
    \hf_0(x,u)=B_2(f(x,u),r_0)\cap\calX,
\end{equation}
where $f$ is the true model, $B_2$ means an $L_2$ ball, and $r_0$ is the radius.
A requirement for the initial uncertain model is that its corresponding maximum feasible zone must be non-empty, i.e., $\rmZ^*(\hf_0)\neq\emptyset$. Otherwise, Algorithm \ref{alg: SEE} will terminate in the first iteration because an empty feasible zone cannot provide any information for updating the uncertain model. To fulfill this requirement, we decrease the uncertainty of $\hf_0$ by setting it to be the true model in a small set near the equilibrium $(x,u)=(0,0)$, i.e.,
\begin{equation}
\begin{aligned}
    \hf_0(x,u)&=\{f(x,u)\}, \\
    \forall x\in B_2(0,r_x)\cap\calX, &\ u\in B_2(0,r_u)\cap\calX,
\end{aligned}
\end{equation}
where $r_x$ and $r_u$ are radii of state and action sets. Different values of $r_0$, $r_x$, and $r_u$ are tested to see their impacts on safe exploration.
\change{It is important to note that this zero-uncertainty region assumption is not a theoretical necessity. Our framework only requires that the initial feasible zone is non-empty so that exploration can commence. In practice, if the initial model is learned from offline data, a non-empty feasible zone can typically be recovered without artificially setting a zero-uncertainty region.}

The distance metric of Lipschitz continuity is chosen as the $L_2$ norm. For a linear dynamic model, the Lipschitz constant $L_f$ can be analytically computed. In this task, we have $L_f=1.73$. However, in most practical control problems, the Lipschitz constant cannot be accurately computed. Instead, we usually choose some large value as its upper bound. Therefore, we also test different values of $L\ge L_f$ to see its impact on safe exploration.

To quantify and compare the model uncertainty during safe exploration, we define a metric called uncertainty degree:
\begin{equation}
    U_{\hf}(x,u)=\left|\hf(x,u)\right|-1,
\end{equation}
which represents the number of ``redundant" states, i.e., states other than the true transition state, in the transition set. The smaller the uncertainty degree, the more accurate the uncertain model is. The true model, which is the most accurate, has an uncertainty degree of 0 at all state-action pairs.

\begin{figure*}[b]
    \centering
    \includegraphics[width=0.9\linewidth]{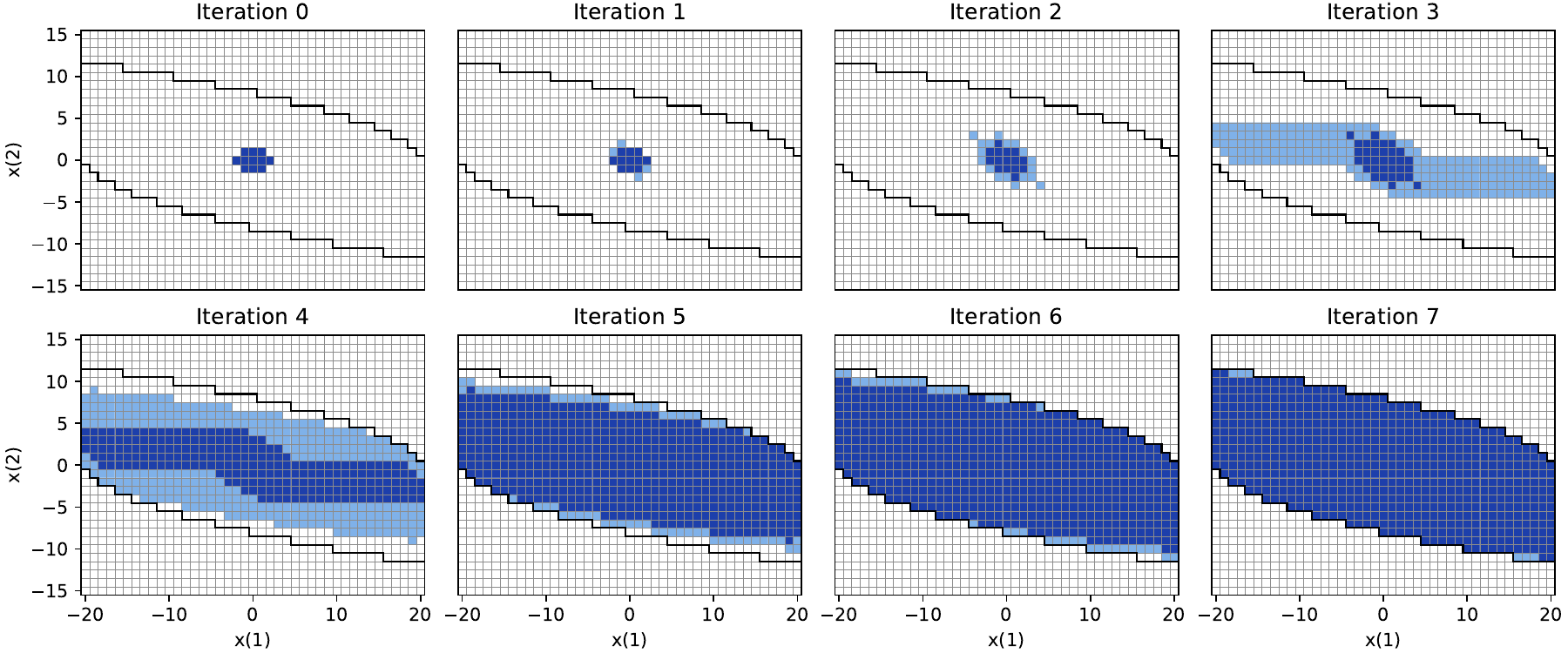}
    \caption{Feasible regions at different iterations on double integrator. The black line is the boundary of the maximum feasible region under the true model. The dark blue region stands for the feasible region in the previous iteration. The light blue region stands for the expanded part of the feasible region in the current iteration.}
    \label{fig: double integrator region}
\end{figure*}

We set the parameters as in Table \ref{tab: double integrator default parameters} and visualize the safe exploration process in Figure \ref{fig: double integrator region} and \ref{fig: double integrator model}.
Figure \ref{fig: double integrator region} shows the feasible region at different iterations, which is the projection of feasible zone on the state space.
\change{The maximum feasible region under the true model, shown by the black thick lines, are computed by the approach from \cite{yang2024synthesizing}. This approach, called feasible region iteration, is a special case of Algorithm \ref{alg: feasible zone iteration}, with the uncertain model replaced by the true model. The update rule is correspondingly simplified to
\begin{equation}
    G_{k+1}(x,u)=c(x)+(1-c(x))\gamma\min_{u'\in\mathcal{U}}G_k(x',u'),
\end{equation}
where $x'=f(x,u)$. The feasible zone computed by this method is the maximum one any algorithm can obtain in theory. We will use it to compute the feasible zone recall of our algorithm.}
The feasible region monotonically expands until the equilibrium of safe exploration is reached after 8 iterations. In this task, the final feasible region equals the maximum feasible region under the true model. However, this is not guaranteed by our safe exploration algorithm. In fact, as shown in the following experiments, this may not be the case in a different task or even in the same task with different parameters.
Figure \ref{fig: double integrator model} shows the model uncertainty degree on each state at different iterations, which is the sum of uncertainty degrees over actions. The uncertainty degree monotonically decreases until the equilibrium of safe exploration is reached. The uncertainty degree inside the feasible region is smaller than that outside the region. Note that the uncertainty degree does not decrease to zero even at convergence, which indicates that the least uncertain model at the equilibrium does not necessarily equal the true model.

\begin{table}[t]
    \centering
    \caption{Default parameter values for double integrator.}
    \label{tab: double integrator default parameters}
    \begin{tabular}{lcc}
        \toprule
        Description & Parameter & Value \\
        \midrule
        Initial transition set radius & $r_0$ & 2.0  \\
        True model state set radius   & $r_x$ & 2.0  \\
        True model action set radius  & $r_u$ & 1.0  \\
        Lipschitz constant            & $L$   & 1.73 \\
        \bottomrule
    \end{tabular}
\end{table}

\begin{figure*}[t]
    \centering
    \includegraphics[width=0.95\linewidth]{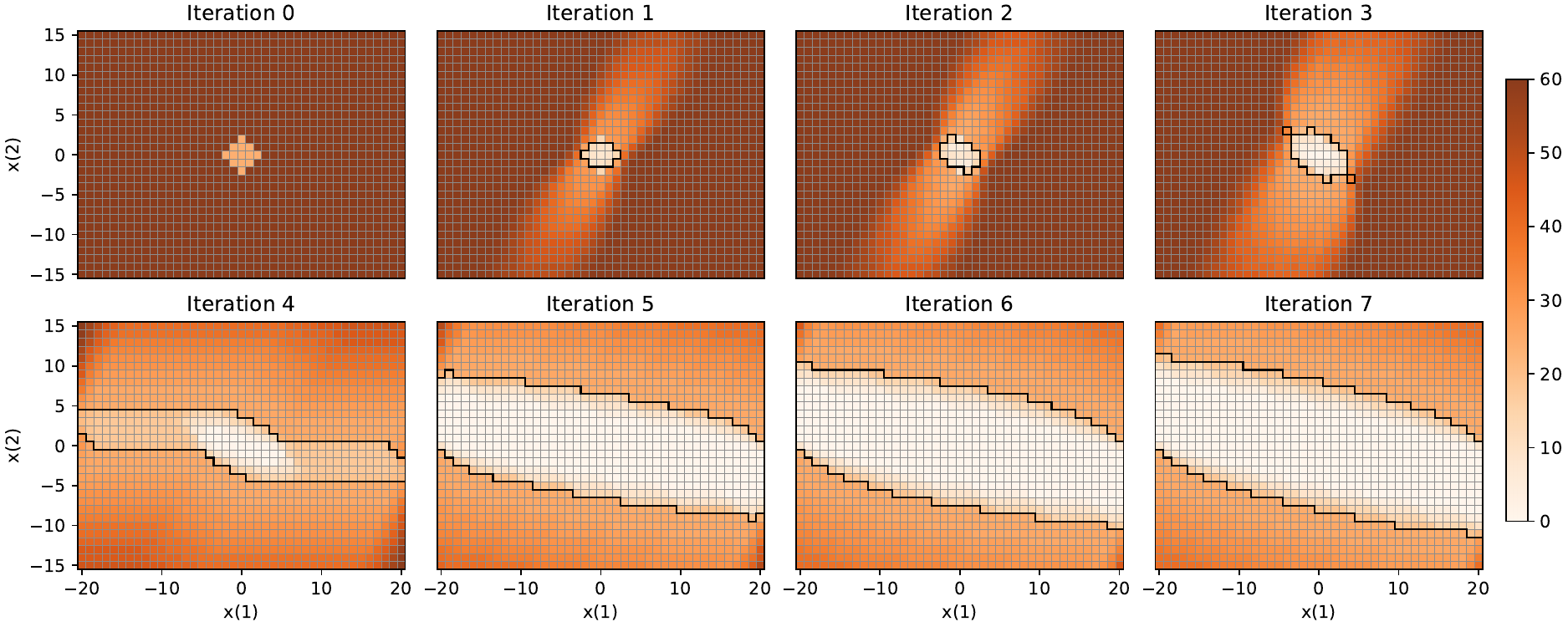}
    \caption{Model uncertainty degree at different iterations on double integrator. The black line is the boundary of the feasible region in the previous iteration. The colors of the grids stand for values of uncertainty degree.}
    \label{fig: double integrator model}
\end{figure*}

\begin{table*}[t]
    \centering
    \caption{Impact of different parameters on feasible region and model uncertainty of double integrator.}
    \label{tab: double integrator statistics}
    \begin{threeparttable}
    \begin{tabular}{cccccccc}
        \toprule
        \multicolumn{4}{c}{Parameter\tnote{1}} & \multicolumn{4}{c}{Result} \\
        \midrule
        $r_0$ & $r_x$ & $r_u$ & $L$ & \# of iterations & Feasible zone recall\tnote{2} \ (\%) & Average UD inside\tnote{3} & Average UD outside\tnote{4} \\
        \midrule
        1.0  & - & - & - & 3  & 100.00 & 0.0 & 2.1 \\
        2.0  & - & - & - & 8  & 100.00 & 0.0 & 5.6 \\
        5.0  & - & - & - & 11 & 100.00 & 0.0 & 29.2 \\
        10.0 & - & - & - & 11 & 100.00 & 0.0 & 33.6 \\
        \midrule
        - & 0.0 & 0.0 & - & 1 & 0.04   & 11.7 & 11.7 \\
        - & 1.0 & 1.0 & - & 2 & 0.44   & 11.2 & 11.3 \\
        - & 2.0 & 1.0 & - & 8 & 100.00 & 0.0  & 5.6 \\
        - & 5.0 & 2.0 & - & 5 & 100.00 & 0.0  & 5.6 \\
        \midrule
        - & - & - & 1.73 & 8  & 100.00 & 0.0  & 5.6 \\
        - & - & - & 1.90 & 10 & 100.00 & 0.0  & 8.4 \\
        - & - & - & 2.00 & 2  & 1.08   & 11.4 & 12.0 \\
        - & - & - & 2.50 & 1  & 0.84   & 11.8 & 12.0 \\
        \bottomrule
    \end{tabular}
    \begin{tablenotes}
        \item[1] Empty parameters take default values in Table \ref{tab: double integrator default parameters}.
        \item[2] Size of explored feasible zone divided by size of the maximum feasible zone under the true model.
        \item[3] Average uncertainty degree (UD) inside the maximum feasible zone.
        \item[4] Average uncertainty degree (UD) outside the maximum feasible zone.
    \end{tablenotes}
    \end{threeparttable}
\end{table*}

\begin{figure}[t]
    \centering
    \includegraphics[width=0.9\linewidth, trim=60 60 0 80, clip]{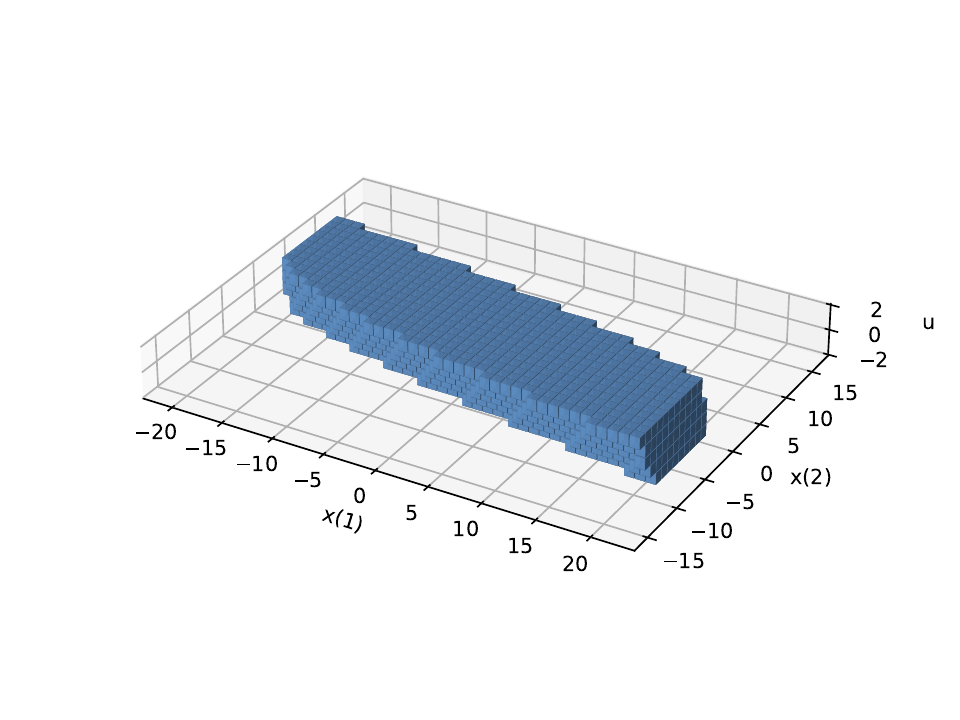}
    \caption{Feasible zone at convergence on double integrator.}
    \label{fig: double integrator zone}
\end{figure}

\begin{figure}[t]
    \centering
    \includegraphics[width=0.35\textwidth, trim=10 15 10 15, clip]{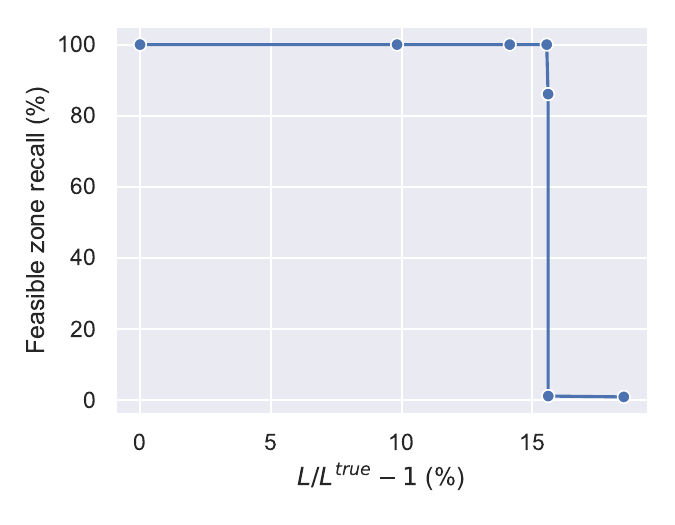}
    \caption{Feasible zone recall versus overall Lipschitz constant $L$.}
    \label{fig: double intergrator FZR vs L}
\end{figure}

We also visualize the feasible zone at convergence in Figure \ref{fig: double integrator zone}. Compared with Figure \ref{fig: double integrator region}, there is an additional $u$ axis for action. Inside the feasible zone, each state corresponds to a set of feasible actions. In the middle area of the feasible zone, all actions are feasible, while in the edge area, only some specific actions are feasible. For example, at state $x=[20,-15]^\top$, the only feasible action is $u=2$. We can intuitively understand through Figure \ref{fig: double integrator zone} that the projection of the feasible zone on the state space (the $x(1)$-$x(2)$ plane in the figure) is the feasible region shown in Figure \ref{fig: double integrator region}.

Then, we change the values of the parameters listed in Table \ref{tab: double integrator default parameters} to see their impacts on feasible region and model uncertainty.
First, we change the value of $r_0$, and the results are shown in the first part of Table \ref{tab: double integrator statistics}. It shows that the value of $r_0$ has little effect on the feasible region and the model uncertainty at convergence. However, as $r_0$ increases, the number of iterations required to reach convergence grows.
Next, we change the values of $r_x$ and $r_u$, and the results are shown in the second part of Table \ref{tab: double integrator statistics}. When $r_x$ and $r_u$ are too small, safe exploration terminates at an early stage with a small feasible region and high model uncertainty. As $r_x$ and $r_u$ increase, the feasible region quickly converges to the maximum one along with low model uncertainty. If $r_x$ and $r_u$ are further increased, the number of iterations required to reach convergence will decrease, indicating that the exploration becomes easier.
Finally, we change the value of $L$, and the results are shown in the third part of Table \ref{tab: double integrator statistics} and in Figure \ref{fig: double intergrator FZR vs L}. When $L$ increases, the number of iterations first increases as well because the update of uncertain model becomes less efficient. As $L$ further increases, the number of iterations decreases, feasible region shrinks, and model uncertainty increases. This indicates that an overly large Lipschitz constant will cause the exploration to terminate at an early stage. In practice, when we do not know the true Lipschitz constant for sure, we shall start from an overestimation and gradually lower the estimated Lipschitz constant, until a reasonable solution can be found. But the estimated value must not be lower than the true one.
\change{Data collected from the real system can help to estimate the true Lipschitz constant. We will discuss this in \ref{sec:Discussion}.}

\begin{figure*}[b]
    \centering
    \includegraphics[width=0.8\linewidth]{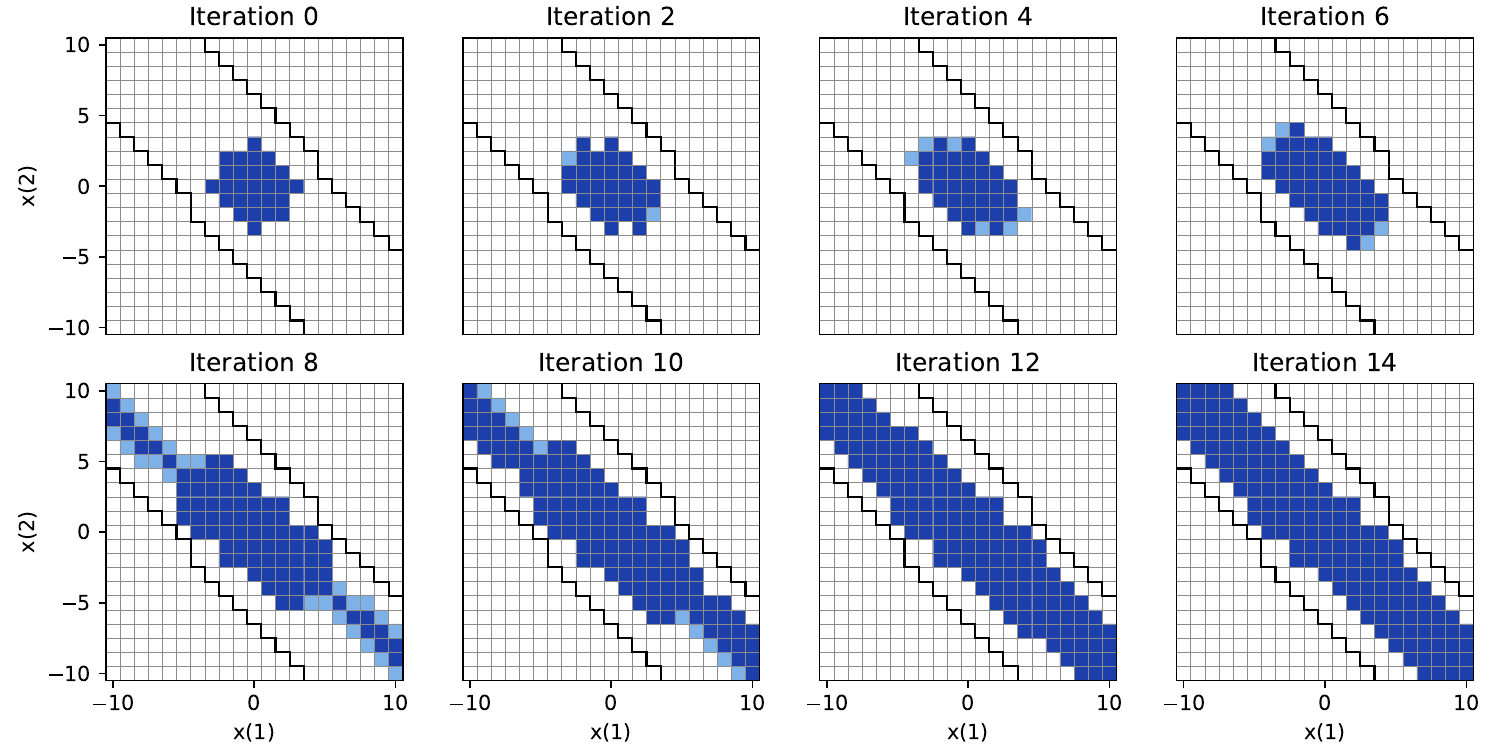}
    \caption{Feasible regions at different iterations on pendulum.}
    \label{fig: pendulum region}
\end{figure*}

\subsection{Pendulum}
Consider a pendulum with the following dynamics:
\begin{equation}
    \dot{\theta}_{t+1}=\dot{\theta}_t+\left(-\frac{3g}{2l}\sin\theta_t+\frac{3}{ml^2}u_t\right)\Delta t,
\end{equation}
where $\theta$ is the angle, $x=[\theta,\dot{\theta}]^\top$ is the state, and the action $u$ is the torque exerted on the pendulum.
\change{The mass of the pendulum $m=1\ \mathrm{kg}$, the length $l=1\ \mathrm{m}$, the gravitational acceleration $g=9.8\ \mathrm{m/s^2}$, and the time step $\Delta t=0.3\ \mathrm{s}$.}
We discretize the dynamics spatially by projecting the state to a grid on a rectangle area:
\change{
\begin{equation}
\begin{aligned}
    \theta&=i\cdot0.05,\ i\in\mathbb{Z}\cap[-10,10], \\
    \dot\theta&=j\cdot0.2,\ j\in\mathbb{Z}\cap[-10,10]. \\
\end{aligned}
\end{equation}
}%
With this discretization, each state corresponds to a unique integer vector $[i,j]$. For simplicity of notation, we use this integer vector to represent the state. The value of action is also restricted to integers.
The state space and action space are:
\begin{equation}
\begin{aligned}
    \calX&=\mathbb{Z}^2\cap[-10,10]\times[-10,10], \\
    \calU&=\mathbb{Z}\cap[-3,3].
\end{aligned}
\end{equation}
The constrained set equals the state space, i.e., $\Xcstr=\calX$.


We set the parameters as in Table \ref{tab: pendulum default parameters} and visualize the safe exploration process in Figure \ref{fig: pendulum region} and \ref{fig: pendulum model}.
The feasible region monotonically expands, and the uncertainty degree monotonically decreases until the equilibrium of safe exploration is reached after 14 iterations.
In this task, the final feasible region is smaller than the maximum feasible region under the true model.
There are two possible reasons for this phenomenon: 1) the explorable maximum feasible region is smaller than the maximum feasible region, and 2) the approximation of the least uncertain model makes the final feasible region smaller than the explorable maximum feasible region.

\begin{table}[t]
    \centering
    \caption{Default parameter values for pendulum.}
    \label{tab: pendulum default parameters}
    \begin{tabular}{lll}
        \toprule
        Description & Parameter & Value \\
        \midrule
        Initial transition set radius & $r_0$ & 3.0  \\
        True model state set radius   & $r_x$ & 3.0  \\
        True model action set radius  & $r_u$ & 2.0  \\
        Lipschitz constant            & $L$   & 3.00 \\
        Lipschitz constant of state   & $L_x$ & 3.00 \\
        Lipschitz constant of action  & $L_u$ & 2.00 \\
        \bottomrule
    \end{tabular}
\end{table}

\begin{figure*}[t]
    \centering
    \includegraphics[width=0.85\linewidth]{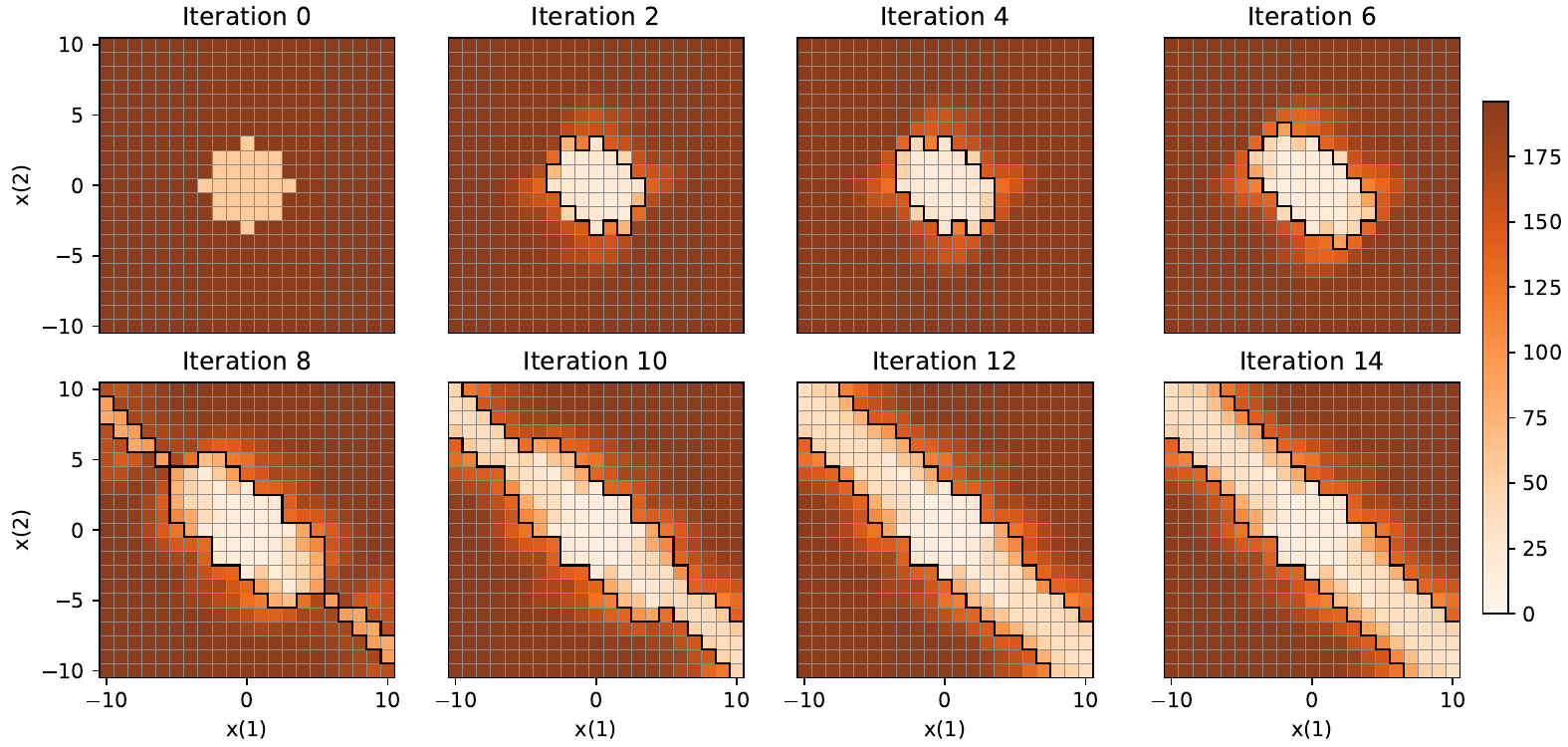}
    \caption{Model uncertainty degree at different iterations on pendulum.}
    \label{fig: pendulum model}
\end{figure*}

\begin{table*}[t]
    \centering
    \caption{Impact of different parameters on feasible region and model uncertainty of pendulum.}
    \label{tab: pendulum statistics}
    \begin{threeparttable}
    \begin{tabular}{cccccccccc}
        \toprule
        \multicolumn{6}{c}{Parameter\tnote{1}} & \multicolumn{4}{c}{Result} \\
        \midrule
        $r_0$ & $r_x$ & $r_u$ & $L$ & $L_x$ & $L_u$ & \# of iterations & Feasible zone recall \ (\%) & Average UD inside & Average UD outside \\
        \midrule
        2.0 & - & - & - & - & - & 3  & 52.50 & 3.4  & 11.2 \\
        3.0 & - & - & - & - & - & 14 & 52.05 & 6.4  & 24.9 \\
        4.0 & - & - & - & - & - & 6  & 10.10 & 37.2 & 45.3 \\
        5.0 & - & - & - & - & - & 6  & 10.10 & 59.3 & 74.0 \\
        \midrule
        - & 2.0 & 1.0 & - & - & - & 1  & 2.55  & 25.7 & 27.6 \\
        - & 3.0 & 1.0 & - & - & - & 3  & 4.55  & 24.1 & 27.3 \\
        - & 3.0 & 2.0 & - & - & - & 14 & 52.05 & 6.4  & 24.9 \\
        - & 4.0 & 2.0 & - & - & - & 7  & 62.71 & 4.70 & 24.0 \\
        \midrule
        - & - & - & 3.00 & 3.00 & 2.00 & 14 & 52.05 & 6.4  & 24.9 \\
        - & - & - & 3.00 & 4.00 & 2.00 & 14 & 52.05 & 6.4  & 24.9 \\
        - & - & - & 3.00 & 3.00 & 3.00 & 5  & 9.88  & 22.6 & 26.8 \\
        - & - & - & 4.00 & 3.00 & 2.00 & 1  & 8.55  & 23.1 & 27.4 \\
        \bottomrule
    \end{tabular}
    \begin{tablenotes}
        \item[1] Empty parameters take default values in Table \ref{tab: pendulum default parameters}.
    \end{tablenotes}
    \end{threeparttable}
\end{table*}

We also visualize the feasible zone at convergence in Figure \ref{fig: pendulum zone}. It shows clearly that the explored feasible zone is a subset of the maximum feasible zone under the true model. In some states inside the latter zone, some or all of the corresponding feasible actions are explored, making these states part of the feasible region in Figure \ref{fig: pendulum region}. In other states, none of the corresponding feasible actions are explored, leaving these states infeasible.

\begin{figure}
    \centering
    \includegraphics[width=0.9\linewidth, trim=60 40 0 60, clip]{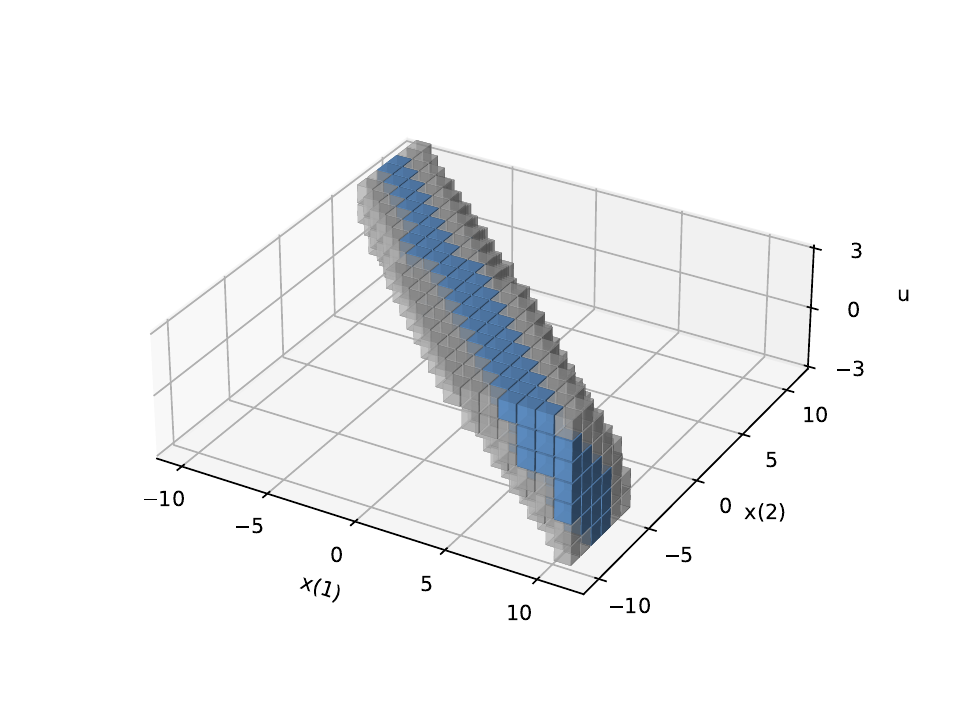}
    \caption{Feasible zone at convergence on pendulum. The blue grids stand for the explored feasible zone. The gray grids stand for the unexplored part of the maximum feasible zone under the true model.}
    \label{fig: pendulum zone}
\end{figure}

We change the values of the parameters listed in Table \ref{tab: pendulum default parameters} to see their impacts on feasible zone and model uncertainty.
First, we change the value of $r_0$, and the results are shown in the first part of Table \ref{tab: pendulum statistics}. As $r_0$ increases, feasible zone shrinks, and model uncertainty increases. This is because a larger $r_0$ means that the initial model is more uncertain, which makes exploration harder.
Next, we change the values of $r_x$ and $r_u$, and the results are shown in the second part of Table \ref{tab: pendulum statistics}. As $r_x$ and $r_u$ increase, the number of iterations first increases and then decreases, the feasible zone expands, and model uncertainty decreases. The reason for this phenomenon is similar to that in double integrator.
Finally, we change the value of $L$, $L_x$, and $L_u$, and the results are shown in the third part of Table \ref{tab: pendulum statistics} and Figure \ref{fig: pendulum FZR vs L}. The change of $L_x$ does not have much effect on feasible region and model uncertainty, while the change of $L$ or $L_u$ has a significant impact.
When $L$ or $L_u$ increases, feasible zone quickly shrinks, and model uncertainty quickly increases. This is because the exploration cannot proceed with too large Lipschitz constants and terminates at an early stage.

\begin{figure*}
    \centering
    \subfloat[]{\includegraphics[width=0.33\textwidth, trim=10 15 10 15, clip]{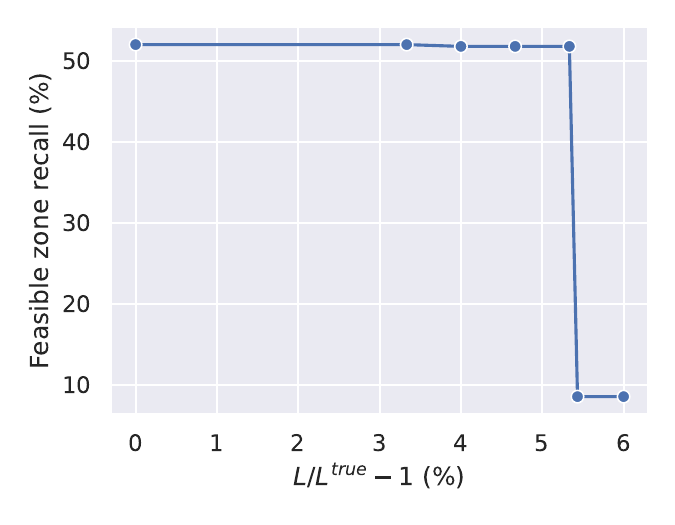}}
    \subfloat[]{\includegraphics[width=0.33\textwidth, trim=10 15 10 15, clip]{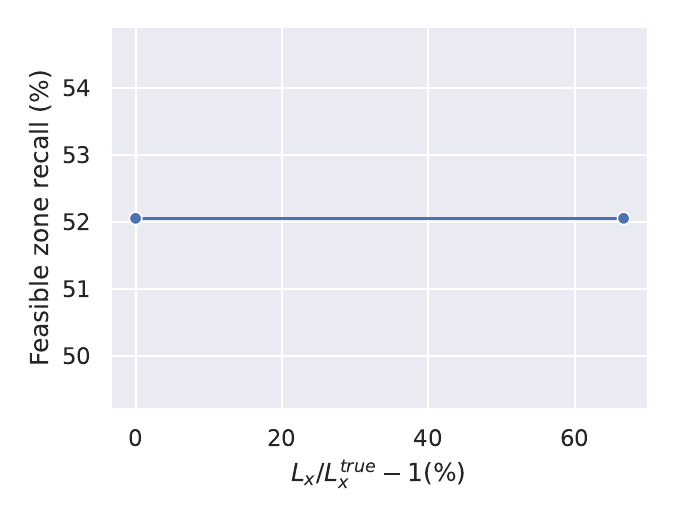}} 
    \subfloat[]{\includegraphics[width=0.33\textwidth, trim=10 15 10 15, clip]{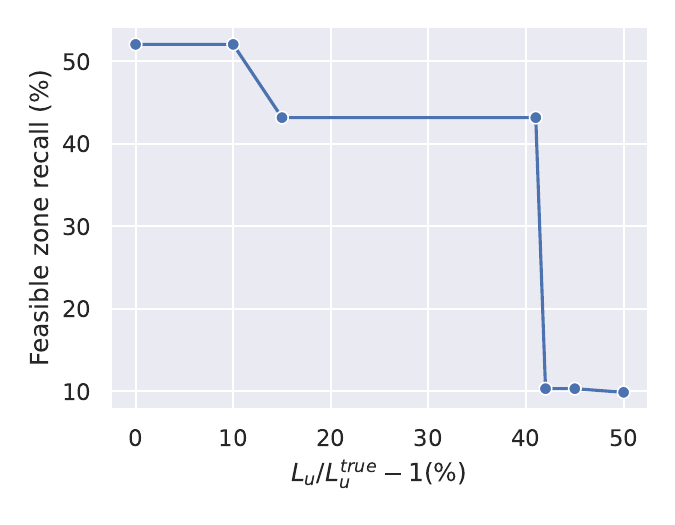}}
    \caption{Feasible zone recall versus (a) overall Lipschitz constant $L$, (b) Lipschitz constant of state $L_x$, (c) Lipschitz constant of action $L_u$.}
    \label{fig: pendulum FZR vs L}
\end{figure*}

\begin{figure*}[b]
    \centering
    \includegraphics[width=0.9\linewidth]{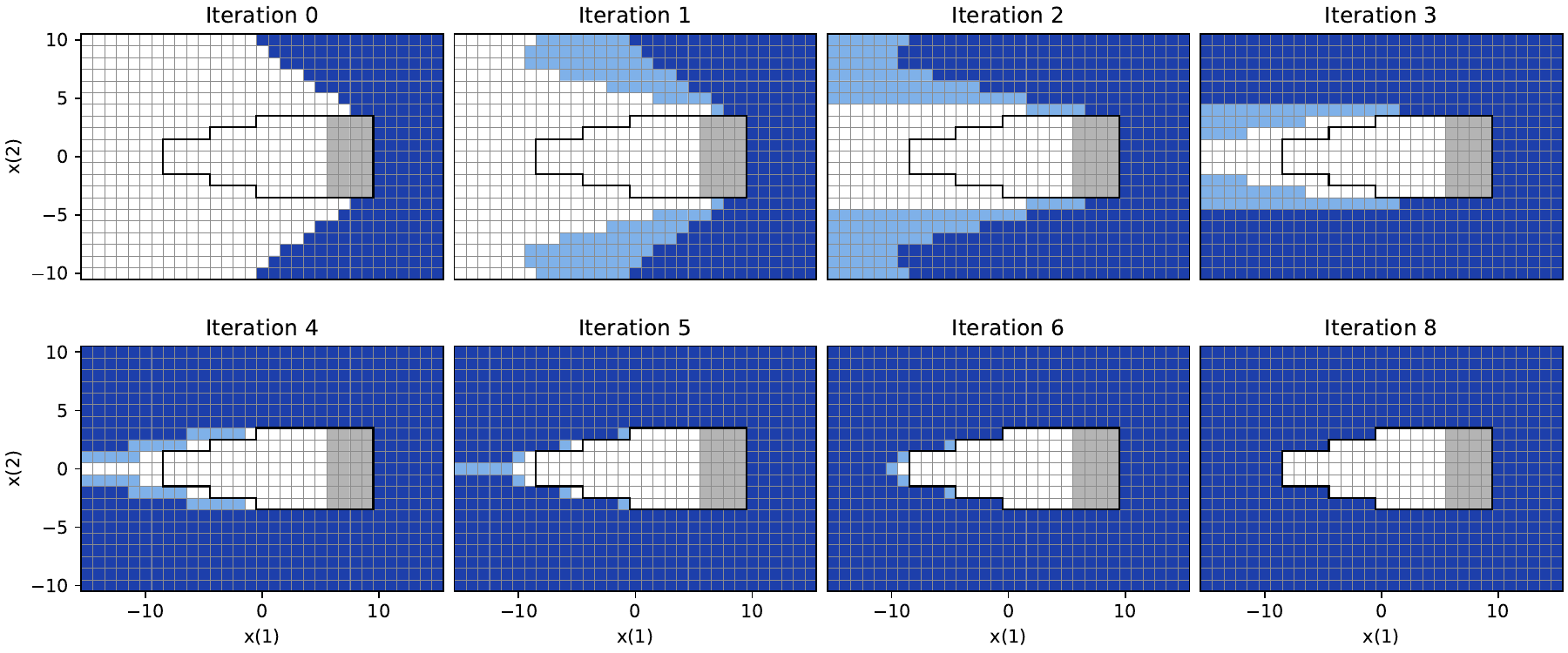}
    \caption{Feasible regions at different iterations on unicycle. The horizontal and vertical axes represents the $y$ and $z$ coordinates, respectively. The heading angle is fixed at $\theta=0$ (heading right). The gray rectangle represents the obstacle.}
    \label{fig: unicycle region}
\end{figure*}

\change{
\subsection{Unicycle}
To test our SEE algorithm beyond 2D tasks, we consider a 3D unicycle with the following dynamics:
\begin{equation}
\begin{bmatrix}
y_{t+1} \\
z_{t+1} \\
\theta_{t+1}
\end{bmatrix}
=
\begin{bmatrix}
y_t \\
z_t \\
\theta_t
\end{bmatrix}
+
\Delta t
\begin{bmatrix}
v \cos\theta_t \\
v \sin\theta_t \\
0
\end{bmatrix}
+
\Delta t
\begin{bmatrix}
0 \\
0 \\
u_t
\end{bmatrix},
\end{equation}
where $(y,z)$ is the position of the unicycle, $\theta$ is the heading angle, $v=1\ \mathrm{m/s}$ is the constant velocity, and $\Delta t=0.2\ \mathrm{s}$. The state is $x=[y,z,\theta]^\top$, and the action $u$ is the angular velocity. We discretize the state and action spaces as:
\begin{align*}
    y&=i\cdot0.05,\ i\in\mathbb{Z}\cap[-15,15], \\
    z&=j\cdot0.05,\ j\in\mathbb{Z}\cap[-10,10], \\
    \theta&=k\cdot\pi/20,\ k\in\mathbb{Z}\cap[-5,5], \\
    u&=l\cdot\pi/8,\ l\in\mathbb{Z}\cap[-2,2]. \\
\end{align*}
The safety constraint requires the unicycle to avoid a rectangular obstacle positioned at $(y,z)\in[6,9]\times[-3,3]$.
}

\begin{figure*}[t]
    \centering
    \includegraphics[width=0.95\linewidth]{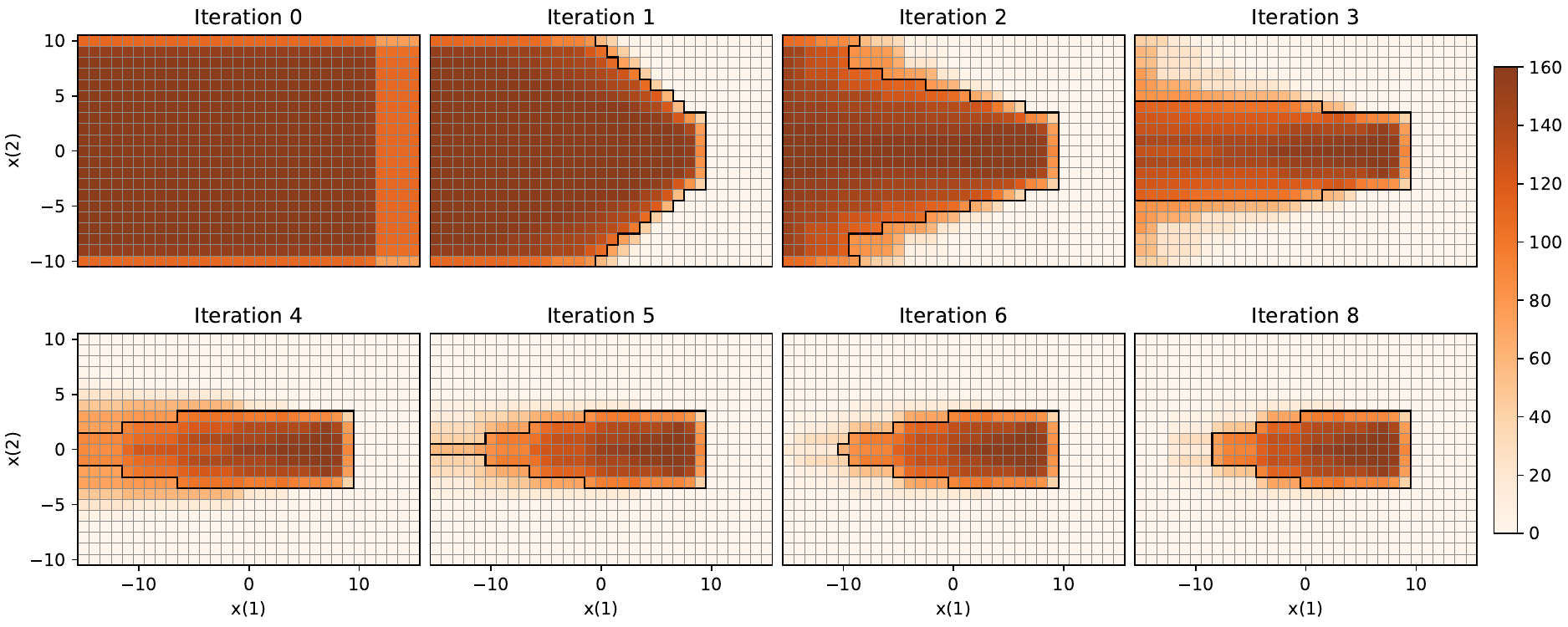}
    \caption{Model uncertainty degree at different iterations on unicycle.}
    \label{fig: unicycle model}
\end{figure*}

\change{
We set the Lipschitz constants $L=3$, $L_x=1.85$, $L_u=1$, and the initial transition set radius $r_0=2$. In this task, we do not assume that the true model is known in a small set.
To visualize the feasible regions and uncertain models in this 3D task, we fix the heading angle $\theta=0$ and show the results in Figure \ref{fig: unicycle region} and \ref{fig: unicycle model}. It shows that the feasible region monotonically expands, and the uncertainty degree monotonically decreases. The equilibrium is reached after 10 iterations with a feasible zone recall of $95.78\%$. The final feasible region visualized on the 2D plane (i.e., the feasible zone sliced at $\theta=0$) equals the maximum feasible region under the true model.
In contrast to the previous two tasks, the feasible zone in this task does not expand from a central region, but from the rightmost boundary of the state space, as shown in the sub-figure of iteration 0. This is because the rightmost states are the furthest from the obstacle and heading towards the opposite position. This makes these states to be easily identified as feasible.
Another phenomenon is that states closer to the obstacle on the $z$ axis are more difficult to be identified as feasible. This is because these states will possibly transfer to a constraint-violating state under the uncertain model.
As the exploration proceeds, the states on the left gradually enters the identified feasible region on the right, eventually covering all feasible states.
These results demonstrate the effectiveness of our algorithm beyond 2D tasks.
}

\change{
\section{Discussion}
\label{sec:Discussion}
In this section, we discuss practical considerations regarding the deployment of the SEE framework in real-world scenarios, addressing the construction of initial models, robustness to model misspecification, scalability to high-dimensional systems, and hyperparameter selection.
\subsection{Construction of initial uncertain models}
The theoretical guarantee of SEE relies on an initial uncertain model $\hf_0$ that contains the true environment dynamics $f$. Constructing such a model in practice typically involves a hybrid approach of physical modeling and data-driven learning.
For systems with known kinematic or dynamic structures, such as robots, the uncertainty often arises from specific physical parameters (e.g., mass, friction, damping). In these cases, domain randomization can be employed to construct $f_0$ by defining reasonable ranges for these parameters, ensuring the true dynamics fall within the randomized set.
For unstructured or unknown environments where physical priors are unavailable, probabilistic machine learning methods can be utilized. Techniques such as Gaussian Processes or ensemble neural networks are capable of quantifying epistemic uncertainty. By calibrating the confidence intervals of these models using initial offline data, one can construct an uncertainty set that covers the true dynamics with high probability.
This requirement aligns with standard assumptions in robust control, where the synthesis of safe controllers necessitates a disturbance support that includes the worst-case realization of the system.
}

\change{
\subsection{Robustness to model misspecification}
While our theoretical analysis assumes the uncertain model is well-calibrated (i.e., $f \in \hf_0$), in practice, this assumption may be violated due to unmodeled dynamics or aggressive approximation.
If the true dynamics lie outside the predicted uncertainty set, the safety guarantee may theoretically be compromised.
However, strict well-calibration is not the only condition for safety. The system remains safe as long as the uncertain model is pessimistic regarding the safety assessment.
Formally, safety is preserved if the worst-case CDF prediction under the uncertain model is lower than the CDF value under the true model:
\begin{equation}
    \max_{\tilde{x}'\in\hf(x,u)}\min_{u'\in\calU}G(\tilde{x}',u') \ge \min_{u'\in\calU}G(f(x,u),u'), \quad \forall (x,u)\in\calZ.
\end{equation}
This condition implies that the feasible zone identified by SEE is a conservative approximation of the true feasible zone. Consequently, even if the model is not perfectly calibrated, the agent will operate within a subset of the safe region, preventing constraint violations.
}

\change{
\subsection{Scalability and generalization}
Prior to Section IV.C, our theoretical framework is established in a general way that applies to both discrete and continuous state-action spaces.
Starting from Section IV.C, we propose two practical algorithm (i.e., Algorithm 2 and 3) for finding maximum feasible zone and least uncertain model respectively.
The requirement of traversing the state-action space makes Algorithm 2 and 3 only apply to discrete spaces.
Despite employing a sufficient condition for removability of the second kind, Algorithm 3 has a time complexity of $\mathcal{O}\sbr{N^3M^3}$, which might be unacceptable for high-dimensional systems or fine discretization.
A potential path toward scalability is through function approximation.
The core components of SEE include a feasible zone represented by an action feasibility function $G:\calX\times\calU\to\R$ and an uncertain model $\hf:\calX\times\calU\to\mathcal{P}\sbr{\calX}$.
The former can be directly approximated by a neural network, which is a common practice in safe RL \cite{yang2024feasibility}.
The later, being an irregular set-valued function, can be approximated in the form of $P:\calX\times\calU\times\calX\to\bbr{0,1}$ which takes as input $\xuxp$ and predicts whether $x^\prime\in\hf\sbr{x,u}$.
In Algorithm 2, it is required to minimize $G$ over $\calU$ and maximize $G$ over $\hf\sbr{x,u}$.
These can be done by learning a safety-oriented policy $\pi:x\mapsto\arg\min_{u}G\sbr{x,u}$ and an adversarial model $f_G:\sbr{x,u}\mapsto\arg\max_{x^\prime\in\hf\sbr{x,u}}G\sbr{x^\prime,\pi\sbr{x^\prime}}$.
}

\change{
\subsection{Estimation of Lipschitz constant}
The estimated Lipschitz constant is a critical prior knowledge required by our proposed SEE.
Prior works \cite{strongin1973convergence,hansen1992using,wood1996estimation} have studied the estimation of the Lipschitz constant of a black box function for a long time.
There are statistical methods for this problem which can be adopted when we have access to an offline dataset $\mathcal{D}=\bbr{\sbr{x_i,u_i,x^\prime_i}}$.
For example, Wood et al. \cite{wood1996estimation} proposed a method based on extreme value theory.
According to this method, we can first samples $n$ pairs of data $\bbr{\sbr{\sbr{x_{1,i},u_{1,i},x^\prime_{1,i}}, \sbr{x_{2,i},u_{2,i},x^\prime_{2,i}}}}_{i=1}^n$ from the offline dataset and compute the maximum slope $l=\max_i\frac{d\sbr{x_{1,i}^\prime, x_{2,i}^\prime}}{d\sbr{\sbr{x_{1,i},u_{1,i}},\sbr{x_{2,i},u_{2,i}}}}$.
Repeat this process $m$ times to get a batch of samples $\bbr{l_1,l_2,...,l_m}$ from the maximum slope distribution.
Finally, we can fit a three-parameter Reverse Weibull distribution to $\bbr{l_1,l_2,...,l_m}$, and the location parameter of the fitted distribution gives an estimation of the Lipschitz constant.
}

\section{Conclusion}
This paper, for the first time, formally defines the goal of safe exploration: to achieve the equilibrium between the feasible zone and the uncertain model.
This subverts the traditional understanding that learning an environment model is a separate task from expanding a feasible zone in safe exploration.
We further propose an equilibrium-oriented safe exploration framework called SEE, which alternates between finding the maximum feasible zone and the least uncertain model.
Through a graph formulation of the uncertain model, we prove that SEE guarantees monotonic refinement of the uncertain model, monotonic expansion of the feasible zone, and convergence to the equilibrium of safe exploration.
Experiments on classic control tasks show that our algorithm successfully expands the feasible zones with zero constraint violation, and achieves the equilibrium of safe exploration within a few iterations.

\bibliographystyle{IEEEtran}
\bibliography{ref}

\begin{IEEEbiography}[{\includegraphics[width=1in,height=1.25in,clip,keepaspectratio]{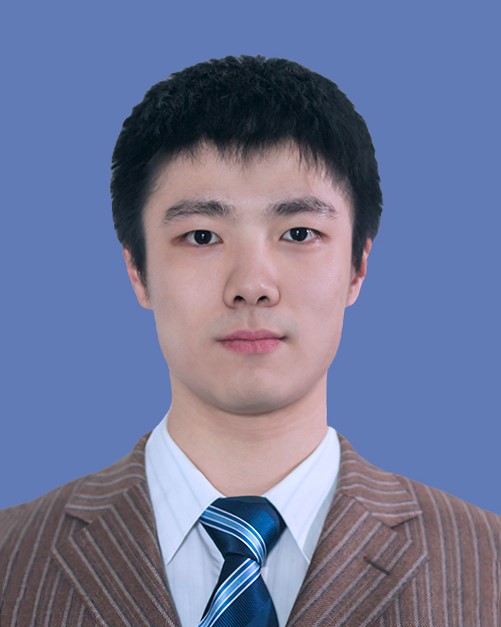}}]{Yujie Yang}
Yujie Yang received his B.S. degree in automotive engineering from the School of Vehicle and Mobility, Tsinghua University, Beijing, China in 2021. He is currently pursuing his Ph.D. degree in the School of Vehicle and Mobility, Tsinghua University, Beijing, China. His research interests include decision and control of autonomous vehicles and safe reinforcement learning.
\end{IEEEbiography}


\begin{IEEEbiography}
[{\includegraphics[width=1in,height=1.25in,clip,keepaspectratio]{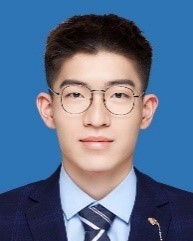}}]{Zhilong Zheng}
Zhilong Zheng received his B.S. degree in automotive engineering from the School of Vehicle and Mobility, Tsinghua University, Beijing, China, in 2022. He is currently pursuing his Ph.D. degree in the School of Vehicle and Mobility, Tsinghua University, Beijing, China. His research interests include decision and control of autonomous vehicles and safe reinforcement learning.
\end{IEEEbiography}


\begin{IEEEbiography}
[{\includegraphics[width=1in,height=1.25in,clip,keepaspectratio]{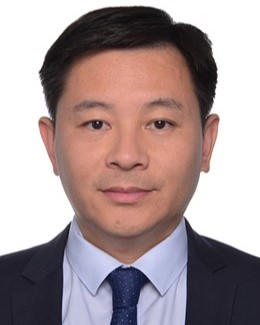}}]{Shengbo Eben Li}
Shengbo Eben Li received his M.S. and Ph.D. degrees from Tsinghua University in 2006 and 2009. Before joining Tsinghua University, he has worked at Stanford University, University of Michigan, and UC Berkeley. His active research interests include intelligent vehicles and driver assistance, deep reinforcement learning, optimal control and estimation, etc. He is the author of over 190 peer-reviewed journal/conference papers, and co-inventor of over 40 patents. Dr. Li has received over 20 prestigious awards, including Youth Sci. \& Tech Award of Ministry of Education (annually 10 receivers in China), Natural Science Award of Chinese Association of Automation (First level), National Award for Progress in Sci \& Tech of China, and best (student) paper awards of IET ITS, IEEE ITS, IEEE ICUS, CVCI, etc. He also serves as Board of Governor of IEEE ITS Society, Senior AE of IEEE OJ ITS, and AEs of IEEE ITSM, IEEE TITS, IEEE TIV, IEEE TNNLS, etc. 
\end{IEEEbiography}

\vfill

\end{document}